\documentclass{article}

\usepackage[margin=1in]{geometry}

\usepackage{graphicx} 
\usepackage{subfigure} 

\usepackage[ruled]{algorithm2e} 
\usepackage{algorithmic}

\usepackage{amssymb,amsfonts,amsmath,amsthm,amscd,dsfont,mathrsfs}
\usepackage{graphicx,float,psfrag,epsfig}
\usepackage{wrapfig}
\usepackage{pgf,pgfarrows,pgfnodes}

\usepackage{hyperref}

\usepackage{url}
\usepackage{multirow}
\usepackage{colortbl}

\usepackage{amsmath}
\usepackage{amssymb}
\usepackage{graphicx}
\usepackage{bm}
\usepackage{dsfont}
\usepackage{paralist}
\usepackage{subfigure}
\usepackage{esvect}
\usepackage{cases}
\usepackage{cleveref}
\crefname{equation}{}{}
\crefrangeformat{equation}{\textnormal{(#3#1#4)--(#5#2#6)}}
\crefname{figure}{}{}
\crefrangeformat{figure}{\textnormal{(#3#1#4)-#5#2#6}}

\usepackage{enumitem}

\allowdisplaybreaks

\newcommand{\separator}{
  \begin{center}
    \rule{\columnwidth}{0.3mm}
  \end{center}
}

\newcommand{\beq}{\begin{eqnarray*}}
\newcommand{\eeq}{\end{eqnarray*}}
\newcommand{\beqn}{\begin{eqnarray}}
\newcommand{\eeqn}{\end{eqnarray}}
\newcommand{\bemn}{\begin{multiline}}
\newcommand{\eemn}{\end{multiline}}

\newcommand{\grad}[1]{\nabla #1}

\newcommand{\note}[1]{[\textcolor{red}{\textit{#1}}]}
\newcommand{\updated}[1]{\textcolor{black}{#1}}

\newcommand{\MF}{\sf}

\DeclareMathOperator{\EXP}{\mathbb{E}}

\newcommand{\minus}{\scalebox{0.4}[0.6]{$-$}}

\newtheorem{theorem}{Theorem}
\newtheorem{lemma}{Lemma}
\newtheorem{proposition}{Proposition}
\newtheorem{property}{Property}
\newtheorem{corollary}{Corollary}
\newtheorem{definition}{Definition}

\begin{document}
\title{\updated{Optimal Inference in Crowdsourced Classification\\via Belief Propagation}
  }
%
%
%

\author{Jungseul~Ok,~
Sewoong~Oh,~
Jinwoo~Shin,~
~and~Yung~Yi
\thanks{
Jungseul Ok, Jinwoo Shin, and Yung Yi are with
the Department of Electrical Engineering at Korea Advanced Institute of Science and Technology,
Daejeon, South Korea (e-mail: \{ockjs, jinwoos, yiyung\}@kaist.ac.kr).
Sewoong Oh is with 
the Department of Industrial and Enterprise Systems Engineering at
University of Illinois at Urbana-Champaign, Illinois, USA
(e-mail: swoh@illinois.edu).
}
  \thanks{This article is partially based on preliminary results published in the proceeding of the 33rd International Conference on Machine Learning (ICML 2016).}
}

\maketitle


\begin{abstract}
  Crowdsourcing systems are popular for solving large-scale labelling
  tasks with low-paid  workers.  We study the
  problem of recovering the true labels from the possibly erroneous
  crowdsourced labels under the popular Dawid-Skene model.
  To address this  inference problem, several algorithms have recently been proposed,
  but the best known guarantee 
  is still significantly larger than the fundamental limit.
We close this gap by 
introducing a tighter lower bound on the
fundamental limit and proving that Belief Propagation (BP) exactly matches this lower bound. 
The
  guaranteed optimality of BP is the strongest in the sense that it is
  information-theoretically impossible for any other algorithm to
  correctly label a larger fraction of the tasks.
Experimental results suggest  that  BP  is 
  close to optimal for all regimes considered and improves upon competing state-of-the-art algorithms.

\end{abstract}



%

\section{Introduction}

Crowdsourcing
platforms provide scalable 
human-powered solutions to labelling large-scale datasets at minimal
cost.  They are particularly popular in domains where the task is easy
for humans but hard for machines, e.g., computer vision and
natural language processing.  For example, the CAPTCHA system
\cite{captcha} uses a pair of scanned images of English words, one for
authenticating the user and the other for the purpose of getting
high-quality character recognitions to be used in digitizing books.
However, 
because the tasks are tedious and the pay is low, 
one of the major issues is the quality of the labels. 
Errors are common even among those who put in efforts. 
In real-world systems,  spammers are abundant, who submit random answers rather
than good-faith attempts to label. There are adversaries 
deliberately giving wrong answers.

A common and powerful strategy to improve
 reliability is to add redundancy: 
assigning each task to multiple workers and aggregating their answers by some algorithm such as 
majority voting. 
Although majority voting is widely used in practice, 
several novel approaches, which outperform majority voting, have been
recently proposed, e.g. \cite{smyth95,jin2002,whitehill09,Welinder10,RYZ10}.  The key
idea is to identify the good workers and give more weights to the
answers from those workers.  Although the ground truths may never be
exactly known, one can compare one worker's answers to those from
other workers on the same tasks, and infer how reliable or trustworthy
each worker is.

The standard probabilistic model for representing the noisy answers in
labelling tasks is the model introduced by Dawid and Skene in \cite{dawid1979}.  
Under this model, the core problem of interest is
how to aggregate the answers to maximize the accuracy of the estimated
labels.  
This is naturally posed as a statistical inference problem
that we call the {\em crowdsourced classification} problem.
Due to the combinatorial nature of the problem, the Maximum A Posteriori (MAP) estimate is optimal but computationally intractable. 
Several algorithms 
have recently been proposed as approximations, and their
performances are demonstrated only by numerical experiments. 
These include algorithms based on spectral methods \cite{GKM11,DDKR13,kos2011,KOS13SIGMETRICS,kos2014}, 
Belief Propagation (BP) \cite{liu2012}, 
 Expectation Maximization (EM) \cite{liu2012,zhang2014}, 
 maximum entropy \cite{ZPBM12,ZLPCS15},
 weighted majority voting \cite{LW89,LYZ13,LY14}, and 
 combinatorial approaches \cite{GZ13}. 

Despite  the algorithmic advances, theoretical advances have been relatively slow. 
Some upper bounds on the performances are known \cite{kos2011,zhang2014,GZ13}, 
but fall short of answering which algorithm should be used in practice. 
In this paper, we ask the fundamental question of whether it is possible to achieve the performance of 
the optimal MAP estimator with a computationally  efficient inference algorithm. 
In other words, we investigate the computational gap between what is information-theoretically possible and 
what is achievable with a polynomial  time algorithm. 

Our main result is that there is no computational gap in the
crowdsourced classification problem for a broad range of problem parameters. 
Under some mild assumptions on the parameters of the problem, we
show the following:
\begin{quote} 
\centering 
	\em Belief propagation is exactly optimal. 
\end{quote}
To the best of our knowledge, our algorithm is the only computationally efficient
approach that provably maximizes the fraction of correctly labeled
tasks, achieving exact optimality.

\vspace{0.05in}
\noindent {\bf Contribution.}
\updated{
We consider binary classification tasks and identify  regimes where the standard BP
achieves the performance of the optimal MAP estimator. 
When each task is assigned enough number of workers, 
we prove that it is impossible for any
other algorithm to correctly label a larger fraction of tasks than BP.
This is the only known algorithm to achieve such a strong notion of
optimality and settles the question of whether there is a
computational gap in the crowdsourced classification problem for a
broad range of parameters.  
}
\updated{ 
  We provide experimental results 
   confirming  the
  optimality of BP 
  for both synthetic and real datasets. 
}

The provable optimality of BP-based algorithms in graphical models
with loops (such as those in our model) is known only in a few
instances including community detection \cite{mossel2014}, error
correcting codes \cite{kudekar2013spatially} and combinatorial
optimization \cite{park2015max}.  Technically, our proof strategy for
the optimality of BP is similar to that in \cite{mossel2014} where
another variant of BP algorithm is proved to be optimal to recover the
latent community structure among users. However, our proof technique
overcomes several unique challenges, arising from the complicated
correlation among tasks that can only be represented by weighted and
directed hyper-edges, as opposed to simpler unweighted undirected
edges in the case of stochastic block models. This might be of
independence interest in analyzing censored block models
\cite{hajek2015exact} with some directed observations.

\vspace{0.05in}
\noindent {\bf Related work.}
\updated{
The crowdsourced classification problem
has been first studied in the {\em dense regime}, where all tasks are assigned all the workers 
\cite{GKM11,zhang2014}.
 In such dense regimes, 
 as the problem size increases, each task  receives increasing number of answers. 
 Thus, previous work has studied the probability of labelling all tasks correctly \cite{GKM11,zhang2014}. 
}
 
In this paper, we focus on the {\em sparse regime}, where each task is
assigned to a few workers.  Suppose $\ell$ workers are assigned each
task.  In practical crowdsourcing systems, a typical choice of $\ell$
is three or five.  For a fixed $\ell$, the probability of error now
does not decay with increasing dimension of the problem.  The
theoretical interest is focused on identifying how the error scales
with $\ell$, that represents how much redundancy should be introduced
in the system.  An upper bound that scales as $e^{-\Omega(\ell)}$
(when $\ell>\ell^*$ for some $\ell^*$ that depends on the problem
parameters) was proved by \cite{kos2011}, analyzing a spectral
algorithm that is modified to use the spectral properties of the
non-backtracking operators instead of the usual adjacency matrices.
This scaling order is also shown to be optimal by comparing it to the
error rate of an oracle estimator.  A similar bound was also proved
for another spectral approach, but under more restricted conditions in
\cite{DDKR13}.  Our main results provide an algorithm that (when
$\ell>C_r$ for some constant $C_r$ depending on $r$ where we denote
the number of tasks per worker by $r$) correctly labels the optimal
fraction of tasks, in the sense that it is information-theoretically
impossible to correctly label a larger fraction for any other
algorithms. 

These spectral approaches are popular due to simplicity, but
empirically do not perform as well as BP.  In fact, the authors in
\cite{liu2012} showed that the state-of-the-art spectral approach
proposed in \cite{kos2011} is a special case of BP with a specific
choice of the prior on the worker qualities. Since the algorithmic
prior might be in mismatch with the true prior, the spectral approach
is suboptimal.

\vspace{0.05in}
\noindent {\bf Organization.}
In Section \ref{sec:problem}, we provide necessary backgrounds
including the Dawid-Skene model for crowdsourced classification and
the BP algorithm.  Section \ref{sec:main} provides the main results of
this paper, and their proofs are presented in Section \ref{sec:proof}.
Our experimental results on the performance of BP are reported in
Section \ref{sec:exp} and we conclude in Section~\ref{sec:conclusion}.

\section{Preliminaries}
\label{sec:problem} 

We describe the mathematical model  and 
present the standard MAP and the  BP approaches. 

\subsection{
Crowdsourced Classification Problem}

We consider a set of $n$ binary tasks, denoted by $V$.  Each task
$i\in V$ is associated with a ground truth $s_i \in \{-1, +1\}$.
Without loss of generality, we assume $s_i$'s are independently chosen
with equal probability.
We let $W$ denote the set of workers who are assigned tasks to answer.
Hence, this task assignment is represented by as a bipartite graph $G
= (V, W, E)$, where edge $(i, u) \in E$ indicates that task $i$ is
assigned to worker $u$.
For notational simplicity, let $N_u : = \{i \in V : (i, u) \in E\}$
denote the set of tasks assigned to worker $u$ and conversely let
$M_i : = \{u \in W : (i, u) \in E\}$ denote the set of workers to
whom task $i$ is assigned.

When task $i$ is assigned to worker $u$, worker $u$ provides a binary
answer $A_{iu}\in \{-1,+1\}$, which is a noisy assessment of the true
label $s_i$.  Each worker $u$ is parameterized by a {\it reliability}
$p_{u} \in [0, 1]$, such that each of her answers is correct with
probability $p_{u}$. 
Namely, for given $p := \{p_u : u \in W\}$, the answers $A := \{A_{iu}
: (i, u) \in E\}$ are independent random variables such that
\begin{align*} 
A_{iu} = 
\begin{cases}
s_i &\text{with probability}\quad p_{u} \\
-s_i &\text{with probability~} 1-p_{u} \\
\end{cases}.
\end{align*}
We assume that the average reliability is greater than $1/2$, i.e.,
$\mu: = \EXP [2p_u - 1] > 0$.

This  Dawid-Skene model is the most popular one in crowdsourcing dating back to \cite{dawid1979}.
The underlying assumption  is that all the tasks share a
homogeneous difficulty;  
the error probability of a worker is consistent across all 
tasks. We assume that the reliability $p_u$'s are i.i.d. according to a  {\it reliability distribution} on $[0, 1]$,
described by a probability density function $\pi$.

For the theoretical analysis, we assume that the bipartite graph is
drawn uniformly over all $(\ell,r)$-regular graphs for some constants
$\ell,r$ using, for example, the configuration model
\cite{bollobas1998}.\footnote{We assume constants $\ell, r$ for
  simplicity, but our results
  hold as long as $\ell r=O(\log n)$.}  Each task is assigned to
$\ell$ random workers and each worker is assigned $r$ random tasks.
In real-world crowdsourcing systems, the designer gets to choose which
graph to use for task assignments. Random regular graphs have been
proven to achieve minimax optimal performance in \cite{kos2011}, and
empirically shown to have good performances. This is due to the fact
that the random graphs have large spectral gaps.

\subsection{MAP Estimator}

Under this crowdsourcing model with given assignment graph $G = (V, W,
E)$ and reliability distribution $\pi$, our goal is to design an
efficient estimator $\hat{s}(A) \in \{-1, +1\}^V$ of the unobserved
true answers $s:= \{ s_i : i \in V\}$ from the noisy answers $A$
reported by workers. In particular, we are interested in the optimal
estimator minimizing the (expected) average bit-wise {\it error rate},
i.e.,
\begin{align} \label{eq:optimization}
\underset{\hat{s}: \text{estimator}}{\text{minimize}} 
\quad P_{\MF err} (\hat{s}(A))
\end{align}
where we define
\begin{align*}
P_{\MF err} (\hat{s}) :=  
 \frac{1}{n} \sum_{i \in V} 
\Pr \left[ 
s_i \neq \hat{s}_i\left(A \right) 
\right].
\end{align*}
The probability is taken with respect to $s$ and $A$ for given
$G$ and $\pi$. From standard Bayesian arguments,   
the maximum a posteriori (MAP) estimator is an optimal solution of \eqref{eq:optimization}: 
\begin{align} 
\hat{s}^*_i(A) := \underset{s_i}{\text{arg} \max} \Pr [s_i   \mid  A ]. \label{eq:optimal-estimator}
\end{align}
However, this MAP estimate is 
challenging to compute, as we show below.  
Note that 
\begin{align}
\Pr[s, p  \mid  A]  %
 &\propto \Pr[p] \cdot \Pr[A \mid s, p] \cr 
&= \prod_{u \in W} \Pr [p_u] \prod_{i \in N_u} \Pr [A_{iu}  \mid  s_i, p_u ] \cr
& = \prod_{u \in W} \pi (p_u) \cdot  p_u^{c_u} (1-p_u)^{r_u - c_u} \label{eq:joint-probability}
\end{align}
where $r_u := |N_u|$ is the number of the tasks assigned to worker $u$
and $c_u := | \{i \in N_u : A_{iu} = s_i\}|$ is the number of the
correct answers from worker $u$. Then, 
\begin{align}
\Pr[s \mid A] &=  \int_{[0,1]^W}  \Pr[s, p \mid A]  d p  \cr %
& \propto    \prod_{u \in W} \underbrace{\int^1_0  \pi (p_u) \cdot  p_u^{c_u} (1-p_u)^{r_u - c_u}  d p_u}_{:= f_u \left(s_{N_u} \right)}
\label{eq:factor-form}
\end{align}
where  we let
$f_u \left(s_{N_u} \right) := \EXP [p_u^{c_u} (1-p_u)^{r_u - c_u}]$
denote the local factor associated with worker $u$. 
We note that the factorized form of the joint
probability of $s$ in \eqref{eq:factor-form} corresponds to a standard  graphical model 
with a  {\it factor  graph} $G = (V, W, E)$ that represents the joint probability of $s$
given $A$, where each task $i \in V$ and each worker $u \in W$
correspond to the random variable $s_i$ and the local factor $f_u$, respectively,
and the edges in $E$ indicate couplings among the variables and the
factors.

 The marginal probability $\Pr[s_i \mid A]$ in the optimal estimator
 $\hat{s}^*_i(A)$ is calculated by marginalizing out $s_{\minus i} :=
 \{s_j : i \neq j \in V \}$ from \eqref{eq:factor-form}, i.e.,
\begin{align}
\Pr [s_i  \mid  A]  %
&= \sum_{s_{\minus i} \in \{\pm1\}^{V\setminus i}} \Pr [s  \mid  A] \cr
&\propto \sum_{s_{\minus i}} \prod_{u \in W} f_u \left(s_{N_u} \right).
\label{eq:marginal-probability}
\end{align}
We note that the summation in \eqref{eq:marginal-probability} is taken
over exponentially many $s_{\minus i} \in \{-1, +1\}^{n-1}$ with
respect to $n$. Thus in general, the optimal estimator $\hat{s}^*$,
which requires to obtain the marginal probability of $s_i$ given $A$
in \eqref{eq:optimal-estimator}, is {\it computationally intractable}
due to the exponential complexity in \eqref{eq:marginal-probability}.


\subsection{Belief Propagation}
Recalling the factor graph described by \eqref{eq:factor-form}, the
computational intractability in \eqref{eq:marginal-probability}
motivates us to use a standard sum-product belief propagation (BP)
algorithm on the factor graph as a heuristic method for approximating
the marginalization. 
The BP
algorithm is described by the following iterative update of messages
$m_{i \to u}$ and $m_{u \to i}$ between task $i$ and worker $u$ and
belief $b_i$ on each task $i$:
\begin{align}
m^{t+1}_{i \to u} (s_i) & \propto \prod_{v \in M_{i} \setminus \{u\}} 
m^t_{v \to i} (s_i) \label{eq:iu-message}\;,\\
m^{t+1}_{u \to i} (s_i) & \propto \sum_{s_{N_u \setminus \{i\}}} f_u (s_{N_u}) \prod_{j \in {N_{u} \setminus \{i\}}} 
m^{t+1}_{j \to u} (s_j) \label{eq:ui-message}\;,\\
b^{t+1}_{i} (s_i ) & \propto \prod_{u \in M_i} m^{t+1}_{u \to i} (s_i) \;,
\label{eq:belief}
\end{align}
where the belief $b_i(s_i)$ is the estimated marginal probability of
$s_i$ given $A$. We here initialize messages with a trivial constant
$\frac{1}{2}$ and normalize messages and beliefs, i.e., $ \sum_{s_i}
m_{i \to u} (s_i) = \sum_{s_i} m_{u \to i} (s_i) = \sum_{s_i} b_i
(s_i) = 1$. Then at the end of $k$ iterations, we estimates the label
of task $i$ as follows:
\begin{align} \label{eq:bp-estimator}
\hat{s}^{{\MF BP}(k)}_i = \underset{s_i}{\arg \max} ~ b^{k}_i(s_i).
\end{align}
We note that if 
the factor graph is a tree, then it is known that the belief
converges, and computes the exact marginal probability
\cite{pearl1982}. 
\begin{property} \label{prop:tree-bp}
  If assignment graph $G$ is a tree so that the corresponding factor
  graph is a tree as well, then 
\begin{align*}
b^{t}_{i} (s_i) = \Pr[s_i \mid A] \quad \text{for all $t \ge n$} 
\end{align*}
where $b^{t}_{i} (s_i)$ is iteratively updated by BP in
\cref{eq:iu-message,eq:ui-message,eq:belief}.
\end{property}
However, for general graphs which may have loops, e.g., random $(\ell,
r)$-regular graphs, BP has no performance guarantee, i.e., BP may
output $b_i (s_i) \neq \Pr[s_i \mid A]$. Further the convergence of BP
is not guaranteed, i.e., the value of $\lim_{t \to \infty} b^t_i (s_i)$
may not exist.

\section{Performance Guarantees of BP}\label{sec:main}

In this section, we provide the theoretical guarantees on the performance of BP.
To this end, we consider the output of BP in \eqref{eq:bp-estimator}
with a choice of $k = \log \log n$. It follows that  the overall
complexity of BP is bounded by $O(n {\ell} {r} \log {r} \cdot \log
\log n)$ as each iteration of BP requires $O(n {\ell} {r} \log
{r})$ operations \cite{liu2012}.

\subsection{Exact Optimality of BP for large $\ell$}
We show in the following that BP is asymptotically optimal 
under a mild assumption that each task is assigned to sufficiently large 
(but constant with respect to the number of tasks) 
number of workers, i.e., $\ell > C_{r, \pi}$. 
This follows from a tighter bound in the non-asymptotic regime, where 
we upper bound the optimality gap, exponentially vanishing in the number of iterations $k$.  We present both results in the following theorem.
\begin{theorem}\label{thm:optimality}
Consider the Dawid-Skene model under the task assignment generated by
a random bipartite $(\ell, r)$-regular graph $G$ consisting of $n$
tasks and $(\ell/r)n$ workers. Let
$\hat{s}^{\MF BP(k)}$ denote the output of BP in
    \eqref{eq:bp-estimator} after $k$ iterations.
 For $\mu := \EXP [2p_u - 1] > 0$, $\EXP[(2p_u - 1)^2] < 1$, and $k \le \log \log n$, 
there exists a constant $C_{r, \pi}$ that only depends on
$\mu$ and $r$ such that if 
$\ell \geq C_{r, \pi}$,
then  for sufficiently large $n$:
	\begin{align}
	\EXP \left[ \min_{\hat{s}: \text{estimator}} P_{\MF err} (\hat{s}) - P_{\MF err} (\hat{s}^{\MF BP(k)})  
 \right]  \; \le \; 2^{-k+1} \;
	\end{align}
		where the expectation
    is taken with respect to the graph $G$. 
\end{theorem}
	As a corollary, it follows that  when we set $k$ increasing with $n$, for example $ k = \log \log n$, we have asymptotic optimality: 
	\begin{eqnarray}
		\lim_{n \to \infty}     \EXP \left[ \min_{\hat{s}: \text{estimator}} P_{\MF err} (\hat{s}) - P_{\MF err} (\hat{s}^{\MF BP})  
 \right] \; =\;  0\;. 
	\end{eqnarray}
A proof is provided in Section~\ref{sec:optimality-thm-pf}. 
Our analysis compares BP to  an {\it oracle} estimator. 
This oracle estimator not only has access to the observed crowdsourced labels, 
but also the ground truths of a subset of tasks. 
Given this extra information, it performs the optimal estimation,
outperforming any algorithm that operates only on the observations.
Using the fact that the random $(\ell, r)$-regular bipartite graph has
a {\it locally tree-like structure} \cite{bollobas1998} and BP is
exact on the local tree \cite{pearl1982}, we prove that the
performance gap between BP and the oracle estimator
vanishes 
due to {\it decaying correlation} from the information on the outside
of the local tree to the root.
This establishes that the gap between BP and the best estimator vanishes, 
in the large system limit.

The assumption on $\mu$ is mild, since it only requires that the crowd
as a whole can distinguish what the true label is. In the case
$\mu<0$, one can flip the sign of the final estimate to achieve the
same guarantee.  
It is more intuitive to understand this assumption as 
formally defining a ground truths, as what the majority crowd would agree on (on average) if we asked the same question to 
all the workers in the crowd. 
Hence, this assumption is without loss of generality. 

The assumption on $\EXP[(2p_u - 1)^2]<1$ is mild, as 
the only case when $\EXP[(2p_u - 1)^2] < 1=1$ is if 
$p_u$ is a a binary random variable taking values only in $\{0,1\}$. 
In such cases, every worker is either telling the exact truths consistently 
or exact the opposite of the truths. 
It follows from Perron-Frobenius theorem \cite{keshavan2009matrix} that 
a naive spectral method would work (and so does several other simple techniques). 
However, BP messages are not smooth in this case, 
which is required for our analysis. 
We believe optimality of BP still holds but requires a  different analysis technique. 

Although practically, BP works well in all regimes of parameters as suggested in Section \ref{sec:exp}, 
theoretically, we require 
 require $k = O(\log \log n)$ to ensure that 
 the graph is locally tree-like  within the neighborhood of depth $k$. 
 Analysis of BP beyond $k=O(\log \log n)$ is an open problem, 
 also in other applications such as community detection \cite{mossel2014}.

When $r=1$, there is nothing to learn about the workers and simple
majority voting is also the optimal estimator.  BP also reduces to
majority voting in this case, achieving the same optimality, and in
fact $C_{1,\pi}=1$.  The interesting non-trivial case is when
$r \ge 2$.  The sufficient condition is for $\ell$ to be larger than
some $C_{\mu,r}$.  Although experimental results in Section
\ref{sec:exp} suggest that BP is optimal in all regimes considered,
proving optimality for $\ell<C_{r,\pi}$ requires new analysis
techniques, beyond those we develop in this paper. The problem of
analyzing BP for $\ell<C_{r,\pi}$ (sample sparse regime) is
challenging. Similar challenges have not been resolved even in a
simpler models\footnote{The stochastic block model is simpler
    than our model in the sense that it has only pair-wise factors  
    which is the special case of our model with $r=2$.} 
    of stochastic
block models, where BP and other efficient inference algorithms have
been analyzed extensively \cite{mossel2014,bordenave2015non}.

\subsection{Relative Dominance of BP for small $\ell$}

For general $\ell$ and $r$, we establish the dominance of BP over two 
existing algorithms with known guarantees: 
the majority voting (MV) and the state-of-the-art iterative algorithm (KOS) in
\cite{kos2011}.
In the sparse regime, where $\ell r=O(\log n)$, these are the only existing algorithms with tight provable guarantees. 
\begin{theorem}\label{thm:superiority}
Consider the Dawid-Skene model under the task assignment generated by
a random bipartite $(\ell, r)$-regular graph $G$ consisting of $n$
tasks and $(\ell/r)n$ workers.
Let $\hat{s}^{\MF MV}$ and $\hat{s}^{\MF KOS}$ denote the outputs of MV and KOS algorithms, respectively.
Then, for any $\ell, r\geq 1$ such that $\ell r = O(\log n)$,
\begin{align*}
\lim_{n \to \infty} \EXP \left[ P_{\MF err} (\hat{s}^{\MF BP}) \right]  ~\le~
 \min \left\{
 \lim_{n \to \infty} \EXP \left[ P_{\MF err} (\hat{s}^{\MF MV}) \right], 
\lim_{n \to \infty}   \EXP \left[ P_{\MF err} (\hat{s}^{\MF KOS}) \right]  \right\} 
\end{align*}
where $\hat{s}^{\MF BP}$ is the output of BP in
\eqref{eq:bp-estimator} with $k = \log \log n$ and the expectations
are taken with respect to the graph $G$.
\end{theorem}

A proof of the above theorem is presented in
Section~\ref{sec:superiority-thm-pf}.  Using
Theorem~\ref{thm:superiority} and the known error rates of MV and KOS
algorithms in \cite{kos2011}, one can derive the following upper bound
on the error rate of BP:
\begin{align}
\lim_{n \to \infty}  \EXP \left[ P_{\MF err} (\hat{s}^{\MF BP}) \right] 
~
\leq
~
\min \Bigg\{  \lim_{n\to \infty} e^{ - \Big(\tfrac{\ell \mu^2}{2} \Big) }, 
\lim_{n\to \infty} e^{ - \Big(  \tfrac{\ell q}{2} \cdot \tfrac{q^2 (\ell-1)(r-1) - 1}{3q^2 (\ell-1)(r-1)+   q(\ell -1)}  \Big)}
\Bigg\}
\label{eq:errorexponent}
\end{align}
where 
$q := \EXP \left[(2p_u -1)^2 \right]$ and all the parameters $\ell,r,\mu$, and $q$ 
can depend on $n$.

This is particularly interesting, since it has been observed
empirically and conjectured with some non-rigorous analysis in
\cite{kos2014} that there exists a threshold $ (\ell-1) (r-1) =
1/q^2$, above which KOS dominates over MV, and below which MV
dominates over KOS (see Figure~\ref{fig:all}).  This is due to the fact
that KOS is inherently a spectral algorithm relying on the singular
vectors of a particular matrix derived from $A$.  Below the threshold,
the sample noise overwhelms the signal in the spectrum of the matrix,
which is known as the spectral barrier, and spectral methods fail.
However, in practice, it is not clear which of the two algorithms
should be used, since the threshold depends on latent parameters of
the problem.  Our dominance result shows that one can safely use BP,
since it outperforms both algorithms in both regimes governed by the
threshold.  This is further confirmed by numerical experiments in
Figure \ref{fig:all}.

\section{Proofs of Theorems}\label{sec:proof}
In this section, we provide the proofs of
Theorems~\ref{thm:optimality}~and~\ref{thm:superiority}. 

\subsection{Proof of Theorem~\ref{thm:optimality}}
\label{sec:optimality-thm-pf}

We first consider the case $r = 1$. Then, $G$ is the set of disjoint
{\it one-level} trees, i.e., star graphs, where the root of each tree
corresponds to task $\rho \in V$ and the leaves are the set $M_\rho$
of workers assigned to the task $\rho$.  Since the graphs are
disjoint, we have $\Pr[s_\rho | A] = \Pr[s_\rho | A_{\rho, 1}]$, where
$A = \{A_{iu} : (i,u)\in E \}$ and $A_{\rho, 1} = \{A_{\rho u} : u \in
M_\rho\}$.  From Property~\ref{prop:tree-bp}, it follows that
\begin{align*}
\hat{s}^{\MF BP}_\rho 
= \underset{s_\rho}{\arg \max} \Pr[s_\rho \mid A_{\rho, 1}]  = \hat{s}^*_\rho(A_{\rho, 1}).
\end{align*}
Therefore, for any $\ell \ge 1$, the optimal MAP estimator
$\hat{s}^*_\rho(A)$ in \eqref{eq:optimal-estimator} is identical
to 
the output $\hat{s}^{\MF BP}_\rho$ with any $k \ge 1$.

From now on, we focus on the case $r \ge 2$, and we condition on a
fixed task assignment graph $G$. 
Define $\rho\in V$ as a random node chosen uniformly at random and let
$\Delta(\hat{s}_\rho)$ denote the gain of estimator
$\hat{s}_\rho$ 
compared to random guessing, i.e.,
\begin{align*} 
\Delta (\hat{s}_\rho) := 
\frac{1}{2} -  \Pr[s_\rho \neq \hat{s}_\rho ]  \text{~and~}
P_{\MF err} (\hat{s}) = \frac{1}{2} -    \Delta (\hat{s}_\rho) 
\end{align*}
where the expectation is taken with respect to the distribution of
$G$. Then it is enough to show that $\Delta(\hat{s}^*_\rho (A))$ and
$\Delta(\hat{s}^{\MF BP}_\rho)$ converge to the same value, i.e., the
limit value of $\lim_{n\to \infty} \EXP[ \Delta(\hat{s}^*_\rho
(A))]$ exists and as $n \to \infty$,
\begin{align} \label{eq:wts}
\EXP \left[ \Delta(\hat{s}^*_\rho (A)) - \Delta(\hat{s}^{\MF BP}_\rho)  \right]  \to 0
\end{align}
where the expectation is taken with respect to the distribution of
$G$.

To this end, we introduce two estimators, $\hat{z}_\rho^*(A_{\rho,
  2k})$ and $\hat{s}_\rho^*(A_{\rho, 2k})$, which have accesses to
different amounts and types of information. 
Let $G_{\rho, 2k} =(V_{\rho, 2k}, W_{\rho, 2k}, E_{\rho, 2k}) $ denote
the subgraph of $G$ induced by all the nodes within (graph) distance $2k$ from
{\it root} $\rho$ and $\partial V_{\rho, 2k}$ denote the set of (task)
nodes\footnote{Since $G$ is a bipartite graph, the distance from task
  $\rho$ to every task is even and the distance from task $\rho$ to every
  worker is odd.}  whose distance from $\rho$ is exactly $2k$. We
now define the following {\it oracle} estimator: 
\begin{align*}
\hat{z}_\rho^*(A_{\rho, 2k}) &:= \underset{s_\rho}{ \arg \max} \Pr [s_i \mid A_{\rho, 2k}, s_{\partial V_{\rho, 2k}} ]
\end{align*}
where we denote
\begin{align} \label{eq:k-hop-info}
A_{\rho,2k} : = \{A_{iu} : (i, u) \in E_{\rho, 2k}\}.
\end{align}

We note that $\hat{z}_\rho^*(A_{\rho, 2k})$ uses the exact label
information of 
$\partial V_{\rho, 2k}$ separating the inside and the outside of
$G_{\rho, 2k}$. Hence one can show that $\hat{z}_\rho^{*}(A_{\rho,
  2k})$ outperforms
the optimal estimator $\hat{s}^*_\rho (A)$. We formally provide the
following lemma whose proof is given in 
Section~\ref{sec:add-lem-pf}.
\begin{lemma} \label{lem:add} Consider the Dawid-Skene model with the
  task assignment corresponding to $G = (V, W, E)$ and let $A$ denote 
  the set of workers' labels. For $\rho \in V$ and $k \ge 1$,
 \begin{align*}
  \Delta (\hat{z}^*_\rho (A_{\rho, 2k})) ~\ge~ \Delta (\hat{z}^*_\rho (A_{\rho, 2k+2})) ~ \dots~\ge~   \Delta (\hat{s}^*_\rho (A)).
 \end{align*}
\end{lemma}
Conversely, if an estimator uses less information than another, it performs worse.
Formally, we provide the following lemma whose proof is given in 
Section~\ref{sec:loss-lem-pf}.
\begin{lemma} \label{lem:loss} Consider the Dawid-Skene model with the
  task assignment corresponding to $G = (V, W, E)$ and let $A$ denote
  the set of workers' labels. For any $\rho \in V$ and subset $A' \subset
  A$,
\begin{align*}
 \Delta (\hat{s}^*_\rho (A)) \ge   \Delta (\hat{s}^*_\rho (A')).
\end{align*}
\end{lemma}

On estimating task $\rho$, BP at $k$-th iteration on $G$ is identical
to BP on $G_{\rho, 2k}$. If $G_{\rho, 2k}$ is a tree, then from
Property~\ref{prop:tree-bp}, BP calculates the exact marginal
probability of $s_\rho$ given $A_{\rho, 2k}$, i.e., 
\begin{align*} 
\hat{s}^{{\MF BP}}_\rho := &~\underset{s_\rho}{\arg \max}~b^{k}_\rho(s_\rho)  
=~\underset{s_\rho}{\arg \max}~\Pr [s_\rho \mid A_{\rho, 2k} ]. 
\end{align*}
Thus,  if $G_{\rho, 2k}$ is a tree, then using Lemmas~\ref{lem:add}~and~\ref{lem:loss} with $A_{\rho, 2k}
\subset A$, we have that
\begin{align}
\Delta( \hat{z}^*_\rho (A_{\rho,2k})) & ~{\ge}~ \Delta( \hat{s}^*_\rho (A)) \cr
& ~{\ge}~ \Delta( \hat{s}^{\MF BP}_\rho ) ~{=}~ \Delta(\hat{s}^*_\rho (A_{\rho, 2k})) \label{eq:main-ineq}
\end{align}
where we define $\hat{s}^*(A_{\rho, 2k}) := \arg \max \Pr [s_\rho \mid
A_{\rho, 2k} ] .$


Consider now a random $(\ell, r)$-regular bipartite graph $G$, which
is a locally tree-like. More formally, from Lemma~5 in \cite{kos2014},
if follows that
\begin{align}  \label{eq:tree-probability}
\Pr[\text{$G_{\rho, 2k}$ is not a tree}] 
\le \frac{3 \ell r }{n} \left( (\ell-1)(r-1) \right)^{2k}.
\end{align}
Hence, by taking the expectation with respect to $G$ and applying
\eqref{eq:tree-probability} to \eqref{eq:main-ineq}, we get
\begin{align}
0 &\le  \EXP \left[\Delta( \hat{s}^*_\rho (A))  - \Delta( \hat{s}^{\MF BP}_\rho ) \right] \cr
& \le \EXP \left[\Delta( \hat{z}^*_\rho (A_{\rho,2k})) - \Delta(\hat{s}^*_\rho (A_{\rho, 2k})) \right]  
+ \frac{3  }{n} (\ell r)^{2k+1} \qquad\label{eq:gain-boud}
\end{align}
where 
the last term in the RHS is less than $2^{-k}$ for sufficiently large
$n$ since we set $k = \log \log n$ and $\ell r = O(\log n)$.  In
addition, from the following lemma, the first term in the RHS is also
less than $2^{-k}$. Hence, this implies \eqref{eq:wts} and the
existence of the limit of
$\lim_{n \to \infty} \EXP [\Delta(\hat{s}^*_\rho(A))]$ due to the
bounded and non-increasing sequence of
$\Delta(\hat{z}^*_\rho(A_{\rho, 2k}))$ in Lemma~\ref{lem:add}. We
complete the proof of Theorem~\ref{thm:optimality}.
\begin{lemma} \label{lem:correlation-decay}
  Suppose $G_{\rho, 2k} = (V_{\rho, 2k}, W_{\rho, 2k}, E_{\rho, 2k})$
  is a tree of which root is task $\rho$ and depth is $2k$, where
  every task except the leaves $\partial V_{\rho, 2k}$ is assigned to
  $l$ workers and every worker labels two tasks. For a given $\mu :=
  \EXP[2p_u - 1] > 0$, there exists a constant $C_{\mu, r}$ such that if
  $\ell \ge C_{\mu, r}$, then for sufficiently large $k$,
\begin{align} \label{eq:correlation-decay}
\left| \Delta(\hat{z}^*_\rho(A_{\rho, 2k})) - \Delta(\hat{s}^*_\rho(A_{\rho, 2k})) \right| \le 2^{-k}.
\end{align}
\end{lemma}
A rigorous proof of Lemma~\ref{lem:correlation-decay} is given in
Section~\ref{sec:correlation-decay-pf}.
Here, we briefly provide the underlying intuition on the proof. As
long as $\mu$ is strictly greater than $0$ and $l$ is sufficiently
large, the majority voting of the one-hop information $\{A_{\rho u} :
u \in M_{\rho}\}$ can achieve high accuracy. On the other hand,
intuitively 
the information in two or more hops is less useful. In the proof of
Lemma~\ref{lem:correlation-decay}, we also provide a quantification of
the {\it decaying rate of the correlation} from the information on
$\partial V_{\rho, 2k}$ to $\rho$ as the distance $2k$ increases.


\subsection{Proof of Theorem~\ref{thm:superiority}} 
\label{sec:superiority-thm-pf}

We note that that KOS is an iterative algorithm where for each $\rho
\in V$ and $k \ge 1$, $\hat{s}^{{\MF KOS}, k}_{\rho}$ depends on only
$A_{\rho, 2k}$ defined in \eqref{eq:k-hop-info}. In addition, it is
clear that MV uses only one-hop information $A_{\rho, 1} \subset
A_{\rho, 2k}$. Hence for given $A_{\rho, 2k}$, the MAP estimator
$\hat{s}^*_\rho(A_{\rho, 2k})$ outperforms MV and KOS, i.e.,
\begin{align} \label{eq:KOS-sub-optimal}
\Delta (\hat{s}^*_\rho(A_{\rho, 2k})) 
\ge \max \left\{ \Delta (\hat{s}^{{\MF MV}}_\rho), \Delta (\hat{s}^{{\MF KOS}, k}_\rho) \right\}.
\end{align}

Recall that if $G_{\rho, 2k}$ is a tree, we have $\hat{s}^{{\MF BP}, k}_\rho =
\hat{s}^*_\rho(A_{\rho, 2k})$. 
Similarly to \eqref{eq:gain-boud}, by taking the expectation with
respect to $G$, it follows that
\begin{align*}
\EXP \left[\Delta( \hat{s}^{{\MF BP}, k}_\rho ) \right] 
~\ge~
 \EXP \left[\max \left\{ \Delta (\hat{s}^{{\MF MV}}_\rho), \Delta (\hat{s}^{{\MF KOS}, k}_\rho) \right\} \right] 
- \frac{3}{n} (\ell r)^{2k+1}
\end{align*}
where the last term goes $0$ as $n \to \infty$ if $\ell r = O(\log n)$
and $k = \log \log n$. This completes the proof of
Theorem~\ref{thm:superiority}.

\section{Proofs of Lemmas} 

\subsection{Proof of Lemma~\ref{lem:add}}
\label{sec:add-lem-pf}

We start with the conditional probability of error given
$A$ in the following: 
\begin{align*}
 \Pr[s_\rho \neq \hat{s}^*_\rho(A) \mid A]  = \min \left\{ \Pr[s_\rho = +1 \mid A], \Pr[s_\rho = -1 \mid A] \right\}.
\end{align*}
This directly implies that
\begin{align}
\Delta(\hat{s}^*_\rho (A)) &=  \EXP \Big[\frac{1}{2} - \Pr[s_\rho \neq \hat{s}^*_\rho (A) \mid A] \Big] \nonumber \\
& =\frac{1}{2} \EXP \Big[ \big| \Pr[s_\rho  = +1 \mid A] - \Pr[s_\rho  = -1 \mid A]  \big|  \Big]. \qquad  \label{eq:alt-delta}
\end{align}
Then, by simple algebra, it follows that
\begin{align*}
\Delta(\hat{s}^*_\rho (A)) 
& = \frac{1}{2}\sum_{A} \Pr[A] \cdot \big| \Pr[s_\rho  = +1 \mid A] - \Pr[s_\rho  = -1 \mid A]  \big|  \\
& = \frac{1}{2}\sum_{A} | \Pr[A, s_\rho = +1] - \Pr[A, s_\rho = -1] | \\
& = \frac{1}{2}\sum_{A} \frac{1}{2} | \Pr[A \mid s_\rho = +1] - \Pr[A \mid s_\rho = -1] |
\end{align*}
where for the last equality we use $\Pr[s_\rho = +1] = \Pr[s_\rho =
-1] = 1/2.$

Let $\phi^{+}_\rho$ denote the distribution of $A$ given $s_\rho =
+1$, and let $\phi^-_\rho$ be the distribution of $A$ given $s_\rho =
-1$, i.e.,
\begin{align*}
\phi^+_i (A) = \Pr[A \mid s_i = +1] \text{~and~} \phi^-_i (A) = \Pr[A \mid s_i = -1].
\end{align*}
Then we have a simple expression of $\Delta(\hat{s}^*_\rho (A))$ as follows:
\begin{align} \label{eq:tv-optimal}
\Delta(\hat{s}^*_\rho (A)) = d_{\tt TV} (\phi^+_\rho, \phi^-_\rho)
\end{align}
where we let $d_{\tt TV}$ denotes the total variation distance, i.e.,
for distributions $\phi$ and $\psi$ on the same space $\Omega$, we define
\begin{align*}
d_{\tt TV } (\phi, \psi) := \frac{1}{2} \sum_{\sigma \in \Omega} |\phi(\sigma) - \psi(\sigma)|.
\end{align*}

Next we note that since $\partial V_{\rho, 2k}$ blocks every path from
the outside of $G_{\rho, 2k}$ to $\rho$, the information on the
outside of $G_{\rho, 2k}$, $A \setminus A_{\rho, 2k}$, is independent
of $s_\rho$ given $s_{\partial V_{\rho, 2k}}$, i.e.,
\begin{align} \label{eq:block}
\Pr[s_\rho \mid A_{\rho, 2k}, s_{\partial V_{\rho, 2k}}] = \Pr[s_\rho \mid A, s_{\partial V_{\rho, 2k}}].
\end{align}
Hence if we set $\psi^+_{\rho, 2k}$ to be the distribution of $A$ and
$s_{\partial V_{\rho, 2k}}$ given $s_\rho = +1$ and similarly for
$\psi^-_{\rho, 2k}$, we have
\begin{align*} 
\Delta(\hat{z}^*_\rho (A_{\rho, 2k})) = d_{\tt TV} (\psi^+_{\rho, 2k}, \psi^-_{\rho, 2k}).
\end{align*}
Noting that $\phi^+_\rho$ and $\phi^-_\rho$ can be obtained by
marginalizing out $s_{\partial V_{\rho, 2k}}$ in $\psi^+_{\rho, 2k}$ and
$\psi^-_{\rho, 2k}$, it follows that
\begin{align} 
d_{\tt TV } (\phi^+_{\rho}, \phi^-_{\rho}) 
&= \frac{1}{2}\sum_{A} |\phi^+_\rho(A)- \phi^-_\rho(A)| \cr
&= \frac{1}{2}\sum_{A} 
\left| \sum_{s_{\partial V_{\rho, 2k}}} 
\left(\psi^+_i(A, s_{\partial V_{\rho, 2k}}) -  \psi^-_i(A, s_{\partial V_{\rho, 2k}}) \right) \right| \cr
&\le \frac{1}{2}\sum_{A} \sum_{s_{\partial V_{\rho, 2k}}} \left|  \psi^+_i(A, s_{\partial V_{\rho, 2k}}) -  \psi^-_i(A, s_{\partial V_{\rho, 2k}}) \right| \cr
&= d_{\tt TV } (\psi^+_{\rho, 2k}, \psi^-_{\rho, 2k}) \label{eq:marginal-inequality}
\end{align}
which implies $\Delta(\hat{z}^*(A_{\rho, 2k})) \ge \Delta(\hat{s}^*(A))$.

We now study $\Delta(\hat{z}^*(A_{\rho, 2k}))$ with different $k$.
Observe that $\partial V_{\rho, 2k}$ blocks every path from $\partial
V_{\rho, 2k+2}$ to $\rho$, i.e., $s_{\partial V_{\rho, 2k+2}}$ is
independent of $s_\rho$ given $s_{\partial V_{\rho, 2k}}$. Thus from
\eqref{eq:block} it follows that
\begin{align*}
\Pr[s_\rho \mid A, s_{\partial V_{\rho, 2k}}] =\Pr[s_\rho \mid A, s_{\partial V_{\rho, 2k}}, s_{\partial V_{\rho, 2k+2}}].
\end{align*}
Therefore, $\psi^+_{\rho, 2k+2}$ and $\psi^-_{\rho, 2k+2}$ can be
obtained from $\psi^+_{\rho, 2k}$ and $\psi^-_{\rho, 2k}$ by
marginalizing out $s_{\partial V_{\rho, 2k+2}}$. Similarly to
\eqref{eq:marginal-inequality}, we have
\begin{align*}
 d_{\tt TV } (\psi^+_{\rho, 2k+2}, \psi^-_{\rho, 2k+2}) \le d_{\tt TV } (\psi^+_{\rho, 2k}, \psi^-_{\rho, 2k})
\end{align*}
which completes the proof of Lemma~\ref{lem:add}.

\subsection{Proof of Lemma~\ref{lem:loss}}
\label{sec:loss-lem-pf}

The proof of Lemma~\ref{lem:loss} is analog to that of
Lemma~\ref{lem:add}.  Let $\varphi^+_{\rho}$ be the distribution of
$A'$ given $s_{\rho} = +1$ and $\varphi^-_{\rho}$ be the distribution
of $A'$ given $s_{\rho} = -1$, i.e.,
\begin{align*} 
\Delta(\hat{s}^*_\rho (A')) = d_{\MF TV} (\varphi^+_\rho, \varphi^-_\rho).
\end{align*}
Since $\varphi^+_{\rho}$ and $\varphi^-_{\rho}$ can be obtained by
marginalizing out $A \setminus A'$ from $\phi^+_{\rho}$ and
$\phi^-_{\rho}$ in \eqref{eq:tv-optimal}, using the same logic for
\eqref{eq:marginal-inequality}, we have
\begin{align*}
d_{\tt TV } (\varphi^+_{\rho}, \varphi^-_{\rho}) \le d_{\tt TV } (\phi^+_{\rho}, \phi^-_{\rho})
\end{align*}
which completes the proof of Lemma~\ref{lem:loss}.



\subsection{Proof of Lemma~\ref{lem:correlation-decay}}
\label{sec:correlation-decay-pf}

We start with several notations which we use in the proof. For $i
\in V_{\rho, 2k}$, let $T_{i} = (V_{i}, W_{i}, E_{i})$ be the subtree
rooted from $i$ including all the offsprings of $i$ in tree $G_{\rho,
  2k}$. We let $\partial V_{i}$ denote the leaves in $T_i$ and $A_{i}
:= \{A_{ju} : (j, u) \in E_{i}\}$. Define
\begin{align*}
X_{i} 
&:= \Pr[ s_i = + 1 \mid  A_{i}]  -  \Pr[ s_i = -1  \mid  A_{i}]
\end{align*}
Here $X_{i}$ is often called the {\it magnetization} of $s_{i}$ given
$A_{i}$.  Similarly, given $A_{i}$ and
$s_{\partial V_i}$, we define the {\it biased} magnetization $Y_i$:
\begin{align*}
Y_{i} 
&:= \Pr[ s_i = + 1 \mid  A_{i}, s_{\partial V_{i}}]  -  \Pr[ s_i = -1  \mid  A_{i}, s_{\partial V_{i}}].
\end{align*}
Using the alternative expression of $\Delta$ in \eqref{eq:alt-delta},
one can check that
\begin{align*}
0 \;\; \le\;\;  \Delta(\hat{z}_i^*(A_{i})) - \Delta(\hat{s}_i^*(A_{i}))  &\; =\;\; \frac{1}{2}\EXP \big[ |Y_{i}| - |X_i| \big]\\
&\; \le \;\; \EXP [|Y_i - X_i|]
\end{align*}
where the expectation is taken with respect to $A_i$ and $s_{\partial
  V_i}$.

Next, for $0 \le t \le k$, we define $i(t) \in \partial V_{\rho,
  2k-2t}$ to be a random node chosen uniformly at random so that
$i(0)$ is a leaf node in $G_{\rho, 2k}$, i.e., $X_{i(0)} = 0$ thus
$|X_{i(0)} - Y_{i(0)}| \le 1$, and $i(k)$ is the root $\rho$, i.e.,
$\Delta(\hat{z}_\rho^*(A_{\rho})) - \Delta(\hat{s}_\rho^*(A_{\rho})) =
\frac{1}{2}\EXP \big[ |Y_{\rho}| - |X_\rho| \big]$. Therefore it is
enough to show that for each $0 \le t < k$
\begin{align} \label{eq:recurrence-wts}
\EXP \left[ {\big|X_{i(t+1)} - Y_{i(t+1)} \big|}  \right]
 \le \frac{1}{2} \EXP \left[ { \big|X_{i(t)} - Y_{i(t)} \big|}  \right]
\end{align}
since this implies 
\begin{eqnarray}
	\EXP \left[ {\big|Y_{\rho} - X_{\rho} \big| }\right]  \;\le\; 2^{-k} \;,
\end{eqnarray} 
and hence $\EXP \left[ \big|Y_{\rho} - X_{\rho} \big|
\right] \to 0$ as $k \to \infty$.
Here $\EXP \left[ {|X_{i(t)} - Y_{i(t)}|} \right]$ quantifies
the correlation from the information at the leaves $\partial V_{i(t)}$
to $s_{i(t)}$. We will show that the correlation exponentially decays with respect to $0 \le t < k$ in what
follows.


\begin{figure}[t]
  \begin{center}
    \begin{minipage}{.55\textwidth}
\centering
      \includegraphics[width=0.4\columnwidth]{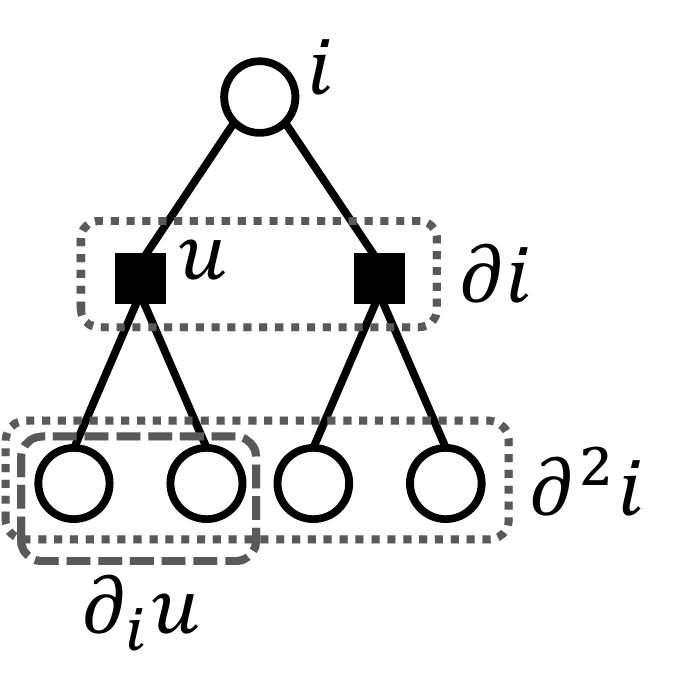}
      \caption{A graphical representation of notations: $\partial
        i, \partial_{i}u,$ and $\partial^2 i$.}
\label{fig:notation}
    \end{minipage}
      \end{center}
\end{figure}

To do so we study certain recursions describing relations among $X$
and $Y$.  Let $\partial i$ be the set of all the offspring of $i$ and
$\partial_i u$ be the set of all the offspring of $u$ in tree $T_i$,
i.e., $\partial i : = \{u \in W_{i} : (i,u) \in E_{i}\}$ and
$\partial_i u : = \{j \in V_{i} : (j,u) \in E_i \}$. (See
Figure~\ref{fig:notation} for a graphical explanation of the
notations.)  Also, define $A_u := \{A_{iu} : (i, u) \in E \}$ and
$\mu_u := (2p_u - 1) \in [-1, 1]$ such that $\mu = \EXP[2p_u -
1] = \EXP[\mu_u] > 0$. Then $f_u$ in \eqref{eq:factor-form} can be
expressed as follows:
\begin{align*}
f_u(s_{N_u}) = \EXP
\left[\prod_{j \in N_u} \frac{1+A_{ju}s_j \mu_u}{2} \right]
\end{align*}
where the expectation is taken for $\mu_u$. Also, using the above
expression of $f_u$ and the fact that $\Pr[s_j \mid A_j] = \frac{1+s_j
  X_j}{2}$, we first write the marginal probability of $s_i$
given $A_u$ and $X_{\partial_{i\!}u}$:
\begin{align*}
\Pr \left[s_i \mid A_{u}, X_{\partial_{i\!} u} \right]& = 
 \sum_{s_{\partial_{i\!} u} } 
 f_u(s_i, s_{\partial_{i\!} u})
  \prod_{j \in \partial_{i\!} u }  \frac{1+s_jX_j}{2}  \\
&=
\sum_{s_{\partial_{i\!} u} } 
\EXP_\mu  \left[\frac{1+A_{iu} s_i \mu_u}{2}
 \prod_{j\in \partial_{i\!} u} \frac{(1+A_{ju}s_j \mu_u)({1+s_jX_j})}{4} \right] \\
&=
\EXP_\mu \left[\frac{1+A_{iu} s_i \mu_u}{2}
 \prod_{j\in \partial_{i\!} u}  \frac{1+ A_{ju} \mu_u X_j }{2}  \right]
\end{align*}
where we let $\EXP_{\mu}$ denote the expectation for only $\mu_u$.
For notational convenience, we define $g^{+}_{iu}$ and $g^{-}_{iu}$ as follows:
\begin{align*}
g^{+}_{iu} (X_{\partial_{i\!} u}; A_{u}) 
&:= \Pr \left[s_i = +1  \mid A_{u}, X_{\partial_{i\!} u} \right] \\
&=
\EXP_{\mu} \left[\frac{1+A_{iu} \mu_u}{2}
 \prod_{j\in \partial_{i\!} u}  \frac{1+ A_{ju} \mu_u X_j }{2}  \right], \\
g^{-}_{iu} (X_{\partial_{i\!} u}; A_{u}) 
&:= \Pr \left[s_i = -1  \mid A_{u}, X_{\partial_{i\!} u} \right] \\
&=
\EXP_{\mu} \left[\frac{1-A_{iu} \mu_u}{2}
 \prod_{j\in \partial_{i\!} u}  \frac{1+ A_{ju} \mu_u X_j }{2}  \right]
\end{align*}
where we may omit $A_u$ in the argument of $g^{+}_{iu}$ and $g^{-}_{iu}$
if $A_u$ is clear from the context.
Using Bayes' rule with $g^{+}_{iu}$ and $g^{-}_{iu}$, we obtain
the following {\it recurrence} for $X$:
\begin{align} \label{eq:recurrence-X}
X_{i} &= h_i(X_{\partial^2 i}) := \frac{\prod_{u \in \partial i} g^+_{iu} (X_{\partial_i u}) -\prod_{u \in \partial i} g^-_{iu} (X_{\partial_i u})}
{\prod_{u \in \partial i} g^+_{iu} (X_{\partial_i u}) +\prod_{u \in \partial i} g^-_{iu} (X_{\partial_i u})}
\end{align}
where we let $\partial^2 i$ denote the set of all the second offspring
of $i$, i.e., $\partial^2 i : = \bigcup_{u \in \partial i} \partial_i
u$.

For simplicity, we focus on a non-leaf/root node $i \in V_{\rho, 2k}$
such that $ i \notin \partial V_{\rho, 2k}$ and $i \neq \rho$ so that
$\left|\partial^2 i \right| = (\ell - 1) \cdot (r - 1)$ and we
consider the case where $s_j = +1$ for all $j$ without loss of
generality since the true label $s_j$ is uniformly distributed and the
choice of $i(t)$ in \eqref{eq:recurrence-wts} is uniform. Then, to
prove \eqref{eq:recurrence-wts}, it is enough to show that
\begin{align}
\EXP^+ \left[ {|X_i - Y_i|}  \right]  
\le \frac{1}{2 (\ell - 1) (r - 1)} \sum_{j \in \partial^2 i}\EXP^+ \left[ {|X_{j} - Y_j|}  \right]
\label{eq:final-wts}
\end{align}
where we let $\EXP^+$ denote the conditional expectation given $s_j =
+1$ for all $j$.

To show \eqref{eq:final-wts}, we will use the mean value theorem. We
first obtain a bound on gradient of $h_i(x)$ for $x \in
[-1,1]^{\partial^2 i}$. Define $g^+_i(x): = \prod_{u \in \partial i}
g^+_{iu}(x_{\partial_{i\!} u})$ and $g^-_i(x): = \prod_{u \in \partial
  i} g^-_{iu}(x_{\partial_{i\!} u})$.  Then, using basic calculus, we
obtain that for $j \in \partial_{i} u$,
\begin{align*}
  \frac{\partial h_i}{\partial x_j}
  &= \frac{\partial }{\partial x_j} \frac{g^+_i -  g^-_i}{g^+_i+ g^-_i}\\
 &= \frac{2}{(g^+_i + g^-_i)^2} \left(
g^-_i \cdot\frac{\partial g^+_i}{\partial x_j}  - g^+_i \cdot \frac{\partial g^-_i}{\partial x_j}
\right)
 \\
 &= \frac{2 g^+_i  g^-_i}
{(g^+_i + g^-_i)^2}
\left(
\frac{1}{g^+_{iu}} \frac{\partial g^+_{iu}}{\partial x_j} %
- \frac{1}{g^-_{iu}} \frac{\partial g^-_{iu}}{\partial x_j} %
\right).
\end{align*}

Using the fact that for $x \in [-1, 1]^{\partial^2 i}$, both $g^+_i$ and $g^-_i$ are
positive, it is not hard to show that
\begin{align} \label{eq:diff-upper}
 \frac{g^+_i g^-_i}{(g^+_i + g^-_i)^2} \le \sqrt{\frac{g^-_i}{g^+_i}}.
\end{align}
We note here that one can replace ${g^-_i}/{g^+_i}$ with ${g^+_i}/{g^-_i}$ in the upper bound.
However, in our analysis, we use \eqref{eq:diff-upper} since we 
focus on the case of $s_i = +1$ where plugging $X_{\partial^2 i}$ or
$Y_{\partial^2 i}$ into $x$ in \eqref{eq:diff-upper},
 $h_i(x)$, which is the magnetization $X_i$
or $Y_i$, will be large thus ${g^-_i}/{g^+_i}$ will be a tighter upper
bound than ${g^+_i}/{g^-_i}$. Our analysis covers all the general
cases because the same analysis with ${g^+_i}/{g^-_i}$ will work with
$s_i = -1$ conversely.

From \eqref{eq:diff-upper}, it follows that for $x \in [-1, 1]^{\partial^2 i}$,
\begin{align*}
\left| \frac{\partial h_i}{\partial x_j} (x) \right| 
\le \left|g'_{ij}(x_{\partial_{i\!}u})  \right| \cdot  \prod_{u' \in \partial i \,:\, u' \neq u} \sqrt{ \frac{g^-_{iu'} (x_{\partial_{i\!}u'})}{g^+_{iu'} (x_{\partial_{i\!}u'})} }
\end{align*}
where we define
\begin{align*}
g'_{ij}(x_{\partial_{i\!}u}) 
 :=
2  \sqrt{ \tfrac{g^-_{iu} (x_{\partial_{i\!}u})}{g^+_{iu} (x_{\partial_{i\!}u})} }
\left(
\tfrac{1}{g^+_{iu} (x_{\partial_{i\!}u})} \tfrac{\partial g^+_{iu}(x_{\partial_{i\!}u})}{\partial x_j} %
- \tfrac{1}{g^-_{iu}(x_{\partial_{i\!}u})} \tfrac{\partial g^-_{iu}(x_{\partial_{i\!}u})}{\partial x_j} %
\right).
\end{align*}
From the assumption on $\mu_u$ (or $p_u$), i.e., $\EXP[\mu_u] > 0$ and  $\EXP[\mu_u^2] < 1$, it follows that for all
$x_{\partial_{i\!} u} \in [-1,1]^{\partial_{i\!}u}$,
$g^+_{iu} (x_{\partial_{i\!} u}) > 0$ and $g^-_{iu} (x_{\partial_{i\!} u}) > 0$.
 Thus, for given $r$, we can find finite $\eta$, which is a constant with respect to $\ell$, such that
\begin{align*}
\max_{x_{\partial_{i\!}u} \in [-1,1]^{\partial_{i\!}u}} |g'_{ij}(x_{\partial_{i\!}u})|
~ \le ~
\eta ~<~ \infty.
\end{align*}
Hence, we have
\begin{align}
\left| \frac{\partial h_i}{\partial x_j} (x) \right| 
\le \eta
 \cdot  \prod_{u' \in \partial i \,:\, u' \neq u} \sqrt{ \frac{g^-_{iu'} (x_{\partial_{i\!}u'})}{g^+_{iu'} (x_{\partial_{i\!}u'})} }.
 \label{eq:diff_bdd}
\end{align}

Let $\EXP^+_{X, Y}$ denote the conditional expectation 
given $X_{\partial^2 i}$, $Y_{\partial^2 i}$, and $s_j = +1$ for all $j$.
Then, using the mean value theorem with \eqref{eq:diff_bdd}, it follows that for given $X_{\partial^2 i}$ and $Y_{\partial^2 i}$, 
there exists $\lambda' \in [0,1]$
such that
\begin{align}
 \EXP^+_{X, Y} |h_i(X_{\partial^2 i}) - h_i(Y_{\partial^2 i})| 
&\le \sum_{u \in \partial {i}}\sum_{j \in \partial_{i} u} 
|X_{j} - Y_{j}| 
\times \EXP^+_{X, Y} \left[ \left|  \frac{\partial h_i}{\partial x_j}  \left(\lambda' X_{\partial^2 i} + (1-\lambda') Y_{\partial^2 i} \right)\right| \right] 
\nonumber  \\
& \le \sum_{u \in \partial {i}}\sum_{j \in \partial_{i} u} 
|X_{j} - Y_{j}| \times
\eta 
\prod_{u' \in \partial i \setminus \{u\}} \max_{\lambda \in [0,1]} \left\{
\EXP^+_{X,Y} \left[
\sqrt{
\tfrac{g^-_{iu'} (\lambda X_{\partial_{i\!}u'}+(1-\lambda) Y_{\partial_{i\!}u'})}{g^+_{iu'} (\lambda X_{\partial_{i\!}u'}+(1-\lambda) Y_{\partial_{i\!}u'})}
}
\right]
\right\}.
 \label{eq:mean-value}
\end{align}
We note that each term in an element of the summation in the RHS is independent to each other.
Thus, from the symmetry among $\{X_{\partial_{i\!} u}\}_{u \in \partial i}$, it follows that
\begin{align}  
\EXP^+ \left[ {|X_i - Y_i|} \right]
~ \le~  
 \sum_{u \in \partial {i}}\sum_{j \in \partial_{i} u} 
\EXP^+ \left[ {|X_{j} - Y_{j}|} \right]  \times
{\eta}
\cdot \left(
\EXP^+ 
\left[ 
\max_{\lambda \in [0,1]} 
\Gamma(\lambda X_{\partial_{i\!}u}+(1-\lambda) Y_{\partial_{i\!}u})
\right]
\right)^{\ell - 1} ~~
\label{eq:exp_bdd}
\end{align}
where we define function $\Gamma(x_{\partial_{i\!} u})$ for given $x_{\partial_{i\!} u} \in [-1,1]^{\partial_{i\!} u}$ 
as follows:
\begin{align*}
\Gamma(x_{\partial_{i\!} u})  & := \EXP^+_{X, Y} \left[ 
\sqrt{\frac{g^-_{iu}(x_{\partial_{i\!} u})}{g^+_{iu}(x_{\partial_{i\!} u})}} \right] \\
& = \sum_{A_u \in \{-1, +1\}^{N_u} }
\Pr{\text{$^+$}}[A_u ] 
\cdot 
\sqrt{
\frac{g^-_{iu} (x_{\partial_{i\!} u};A_u)}{g^+_{iu} (x_{\partial_{i\!} u};A_u)}
}
 \\
& = \sum_{A_u \in \{-1, +1\}^{N_u} } 
\EXP_{\mu} \left[ \prod_{j \in N_u} \frac{1+A_{ju} \mu_u}{2}\right]
 \times
\sqrt{
\frac{
\EXP_{\mu} \left[\frac{1-A_{iu} \mu_u}{2}
 \prod_{j\in \partial_{i\!} u}  \frac{1+ A_{ju} \mu_u x_j }{2}  \right] 
}
{
\EXP_{\mu} \left[\frac{1+A_{iu} \mu_u}{2}
 \prod_{j\in \partial_{i\!} u}  \frac{1+ A_{ju} \mu_u x_j }{2}  \right] }
}
\end{align*}
where we let $\Pr{\text{$^+$}}$ denote the conditional probability measure given that $s_j$ for all $j$.

We obtain a bound of the last term of \eqref{eq:exp_bdd} in the following lemma whose proof
is presented in Section~\ref{sec:one-lem-pf}.
\begin{lemma} \label{lem:one}
 For given $\pi$ such that $\mu : =  \EXP [\mu_u] > 0$ and $\EXP[\mu_u^2] < 1$, 
 there  exists constant $C'_{\pi, r}$such that for any $\ell \ge C'_{\pi, r}$,
\begin{align*}
 \EXP^+ \left[
 \max_{\lambda \in [0, 1]} 
\Gamma(\lambda X_{\partial_{i\!}u}+(1-\lambda) Y_{\partial_{i\!}u})
\right]
~\le~ 
\sqrt{1-\frac{\mu^2}{4}}
~<~ 1.
\end{align*}
\end{lemma}

Using the above lemma, we can find a sufficiently large constant $C_{\pi, r} \ge C'_{\pi, r} $ such that if $\ell-1 \ge
C_{\pi, r}$,
\begin{align*}
\eta \left(1-\varepsilon_{\pi, r} \right)^{{C_{\pi, r}}}
 ~\le~ \frac{1}{2C_{\pi, r}(r-1)}
  ~\le~ \frac{1}{2(\ell - 1)(r-1) }
\end{align*}
which implies \eqref{eq:final-wts} with \eqref{eq:exp_bdd} and completes the proof of
Lemma~\ref{lem:correlation-decay}.

\subsection{Proof of Lemma~\ref{lem:one}}
\label{sec:one-lem-pf}

We first obtain a bound on $X_j$ and $Y_j$ for  $j \in \partial_i u$.
The MAP estimator
$\hat{s}^*_j(A_j)$ of $s_j$ given $A_j$ is identical to estimating
$s_j = +1$ if $X_j$ is positive and $s_j = -1$ otherwise. From the definition of the MAP estimator,
it is straightforward to check
\begin{align*}
\Pr[s_j \neq \hat{s}^*_j(A_j)] = \frac{1-\EXP^+[X_j]}{2}.
\end{align*}
In addition, as Lemma~\ref{lem:loss} states, the
MAP estimator $\hat{s}^*_j(A_j)$ outperforms MV with $\{ A_{ju(jj')} :
j' \in \partial j\}$. Using Hoeffding's
bound, the error probability of MV is bounded as follows:
\begin{align*}
\frac{1-\EXP^+[X_j]}{2}
& ~\le~ \Pr\text{$^+$}[ s_j \neq \hat{s}^{\MF MV}_j]  \\
& ~\le~  \exp \left(- \frac{(|\partial j|-1)\mu^2}{2} \right)
\end{align*}
where Lemma~\ref{lem:loss} implies the first inequality.
Similarly,
$\hat{z}^*_j(A_j)$ of $s_j$ given $A_j$ and $\partial V_i$ is
identical to estimating $s_j = +1$ if $Y_j$ is positive and $s_j = -1$
otherwise. 
Using Lemma~\ref{lem:add} and 
 the Markov inequality, it follows that for small $\varepsilon > 0$, 
\begin{align} \label{eq:XY_bound}
\Pr\text{$^+$}[Y_j < 1-\varepsilon] & ~\le~  \Pr\text{$^+$}[X_j < 1-\varepsilon] \nonumber \\
&~\le~ \frac{2\exp\left(- \frac{(|\partial j|-1)\mu^2}{2} \right)}{\varepsilon}
\end{align}
where we use Lemma~\ref{lem:add} for the first inequality and the Markov inequality for the second one.

Since $0 < \EXP[\mu_u] $ and $\EXP[\mu_u^2] < 1$, we can find  finite constants
$\eta'$ and $\eta''$ such that for all $x \in [0,1]^{\partial_{i\!} u}$,
\begin{align*}
\left|  \Gamma(x) \right| \le \eta' \quad \text{and} \quad
\left| \frac{\partial \Gamma(x)}{\partial x_j} \right| \le \eta''.
\end{align*}
Let $\varepsilon(\ell) := \exp\left(- \frac{(\ell-1)\mu^2}{4} \right) \le \exp\left(- \frac{(|\partial j|-1)\mu^2}{4} \right).$
Then, we have
\begin{align*}
\EXP^+ \left[
 \max_{\lambda \in [0, 1]} 
\left\{
\Gamma(\lambda X_{\partial_{i\!}u}+(1-\lambda) Y_{\partial_{i\!}u})
\right\}
\right]
&\le
\left(1 - \Pr{\text{$^+$}}[X_{j} > 1-\varepsilon \text{~and~} Y_j > 1-\varepsilon ~\forall j \in \partial_{i}u] \right) 
 \times \max_{x \in [-1,1]^{\partial_{i\!} u}} \Gamma(x) \\
& + \Pr{\text{$^+$}}[X_{j} > 1-\varepsilon \text{~and~} Y_j > 1-\varepsilon ~\forall j \in \partial_{i}u]
 \times  \max_{x \in [1-\varepsilon,1]^{\partial_{i\!} u}}\Gamma(x) \\
&\underset{(a)}{\le}
\left( \sum_{j \in \partial_{i}u} \Pr{\text{$^+$}}[X_{j} \le 1-\varepsilon] + \Pr{\text{$^+$}}[ Y_j \le 1-\varepsilon] \right) 
 \times \max_{x \in [-1,1]^{\partial_{i\!} u}} \Gamma(x)
\\
&~~+ 1\times \max_{x \in [1-\varepsilon,1]^{\partial_{i\!} u}}\Gamma(x)  \\
&\underset{(b)}{\le}
4 r \eta' \varepsilon(\ell)
+ \max_{x \in [1-\varepsilon,1]^{\partial_{i\!} u}}\Gamma(x)  \\ 
&\underset{(c)}{\le}
4 r \eta' \varepsilon(\ell)
+ \Gamma(1) + \varepsilon(\ell) \eta''
\end{align*}
where we use the union bound, \eqref{eq:XY_bound}, and the mean value theorem for (a), (b), and (c), respectively.
Since $\varepsilon(\ell)$ decreases as $\ell$ increases, it is enough to show $\Gamma(1) \le \sqrt{1-\mu^2}$.
Using the Cauchy-Schwarz inequality, it follows that 
\begin{align*}
\Gamma(1_{\partial i\!} u) 
&= \sum_{A_{\partial_{i\!}u} } \!
\sqrt{\EXP_{\mu}\! \Big[ \tfrac{1+\mu_u}{2} \! \prod_{j \in \partial_{i\!} u}  \!\tfrac{1+A_{ju} \mu_u}{2}\Big] }
\! \cdot  \! \sqrt{
\EXP_{\mu}\! \Big[ \tfrac{1-\mu_u}{2} 
\! \prod_{j \in \partial_{i\!} u} \!
 \tfrac{1+A_{ju} \mu_u}{2}\Big]
} \\
&
+
\sum_{A_{\partial_{i\!}u} } \!
\sqrt{\EXP_{\mu} \!\Big[ \tfrac{1-\mu_u}{2}
\! \prod_{j \in \partial_{i\!} u}  \!
\tfrac{1+A_{ju} \mu_u}{2}\Big] }
\! \cdot  \!
\sqrt{
\EXP_{\mu} \! \Big[ \tfrac{1+\mu_u}{2}
\! \prod_{j \in \partial_{i\!} u}
\! \tfrac{1+A_{ju} \mu_u}{2}\Big]
}
\\
&
\le
\sqrt{\frac{1+\mu}{2}}
\cdot 
\sqrt{\frac{1-\mu}{2}}  + \sqrt{\frac{1-\mu}{2}}
\cdot 
\sqrt{
\frac{1+\mu}{2}} = \sqrt{1-\mu^2}.
\end{align*}
This completes the proof.

\section{Experimental Result}
\label{sec:exp}

In this section, we evaluate the performance of BP using both synthetic
datasets and real-world Amazon Mechanical Turk datasets to study how our
theoretical findings are demonstrated in practice.

\subsection{Tested Algorithms}
We compare BP and a variant of BP to two oracle algorithms and
several state-of-the-art algorithms in \cite{dawid1979, kos2011,
  liu2012}, each of which are briefly summarized next.

\vspace{0.05in}
\noindent{\bf A practical version of BP.} 
We note that BP, named {\MF BP-True} in our plots, requires the
knowledge of the prior on $p_u$'s.
However, in practice, the distribution is typically unknown.  Thus, we
design a practical version of BP, which we call {\MF EBP} (Estimation
and Belief Propagation) that has an additional procedure that extracts
the required statistics on the prior of $p_u$'s from the observed
data.  In {\MF EBP}, starting with a certain initialization of labels,
it first estimates the statistics of each worker's reliability
assuming the labels are true, and updates the labels via BP using the
estimated statistics as the reliability distribution, over multiple
rounds in an iterative manner.
We will focus on two versions of {\MF EBP} with one and two rounds,
respectively, marked as {\MF EBP(1)} and {\MF EBP(2)}, which is motivated by our
empirical observation that two rounds are enough to achieve good
performance, and the gain from more rounds is marginal.





\vspace{0.05in}
\noindent {\bf Oracle algorithm.} 
Since computing the MAP estimate is computationally intractable, we
instead compute the lower bound on the error rate, using the following
estimator with access to an oracle.  We consider an oracle MAP estimator
which has an omniscient access to a subset of the true labels of tasks
to label each task. We consider the {\MF Oracle-Task} that, to estimate task
$\rho$, uses the true labels of the only tasks separating the inside
and the outside of the breadth-first searching tree rooted from task
$\rho$ in $G$. Then due to the exactness of BP on a tree in
Property~\ref{prop:tree-bp} and Lemma~\ref{lem:add}, we can obtain the
lower bound in a polynomial time.

\vspace{0.05in}
\noindent {\bf Tested algorithms for comparisons.}
For comparison to the state-of-the-art algorithms, we test the majority
voting ({\MF MV}), an iterative algorithm ({\MF KOS}) \cite{kos2011}), the expectation
maximization ({\MF EM}) \cite{dawid1979})
and an approach based on approximate mean field ({\MF AMF})
\cite{liu2012}).  Specifically, as the authors in \cite{liu2012} suggested, we run
{\MF EM} and {\MF AMF} with $\text{Beta}(2, 1)$ as the input
distribution on workers' reliability.

We terminate all algorithms that run in an iterative manner
(i.e., all the algorithms except for {\MF MV}) at the maximum of $100$
iterations or with $10^{-5}$ message convergence tolerance, 
all results are averaged on $100$ random samples.




\begin{figure*}[t!]
   \begin{center}
\begin{minipage}{.32\textwidth}
   \begin{center}
     \subfigure[SH model with  $r = 5$\label{fig:SH_r5}]{\includegraphics[width=1\columnwidth]{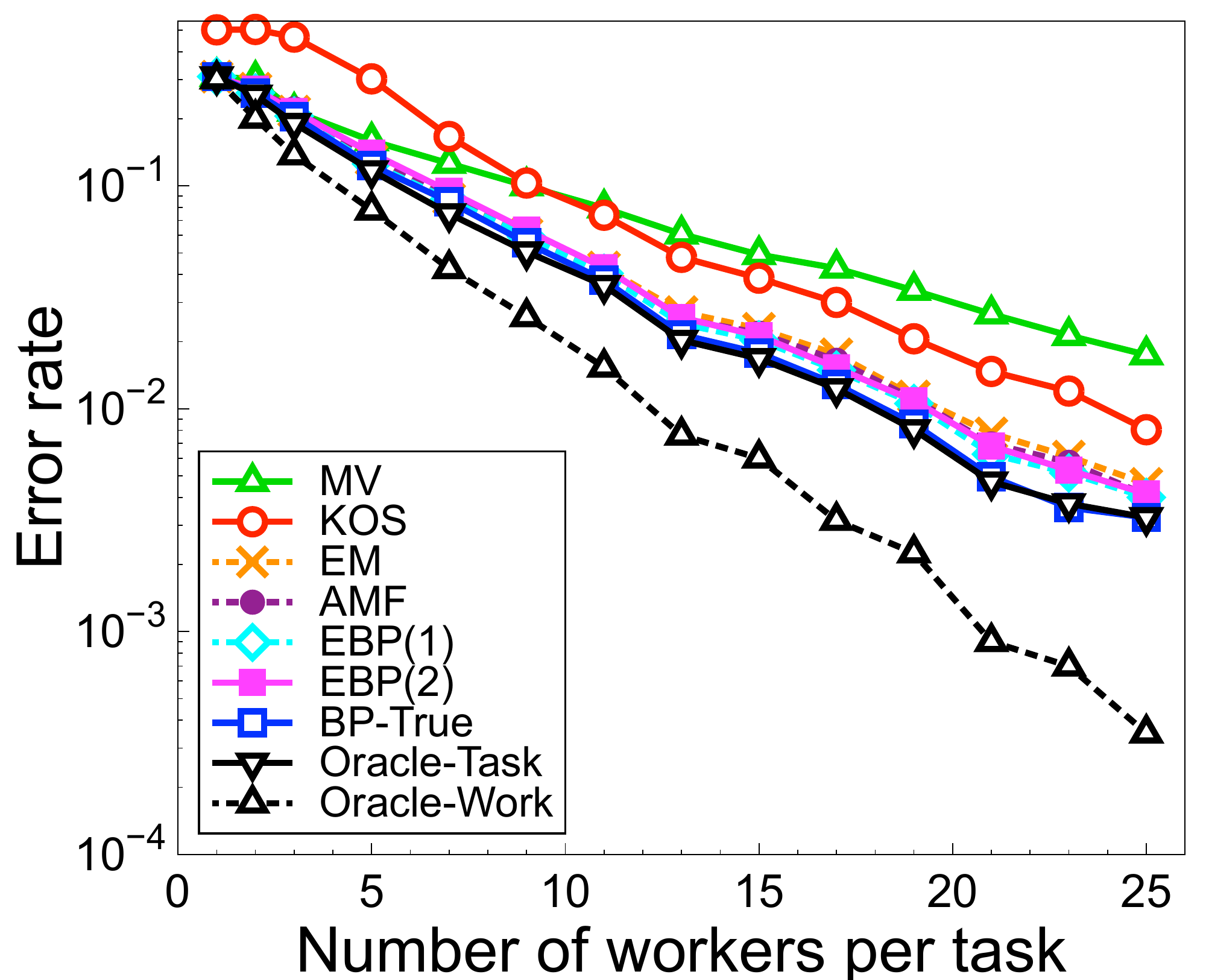}
     } \\
     \subfigure[SH model with  $\ell = 5$\label{fig:SH_l5}]{\includegraphics[width=1\columnwidth]{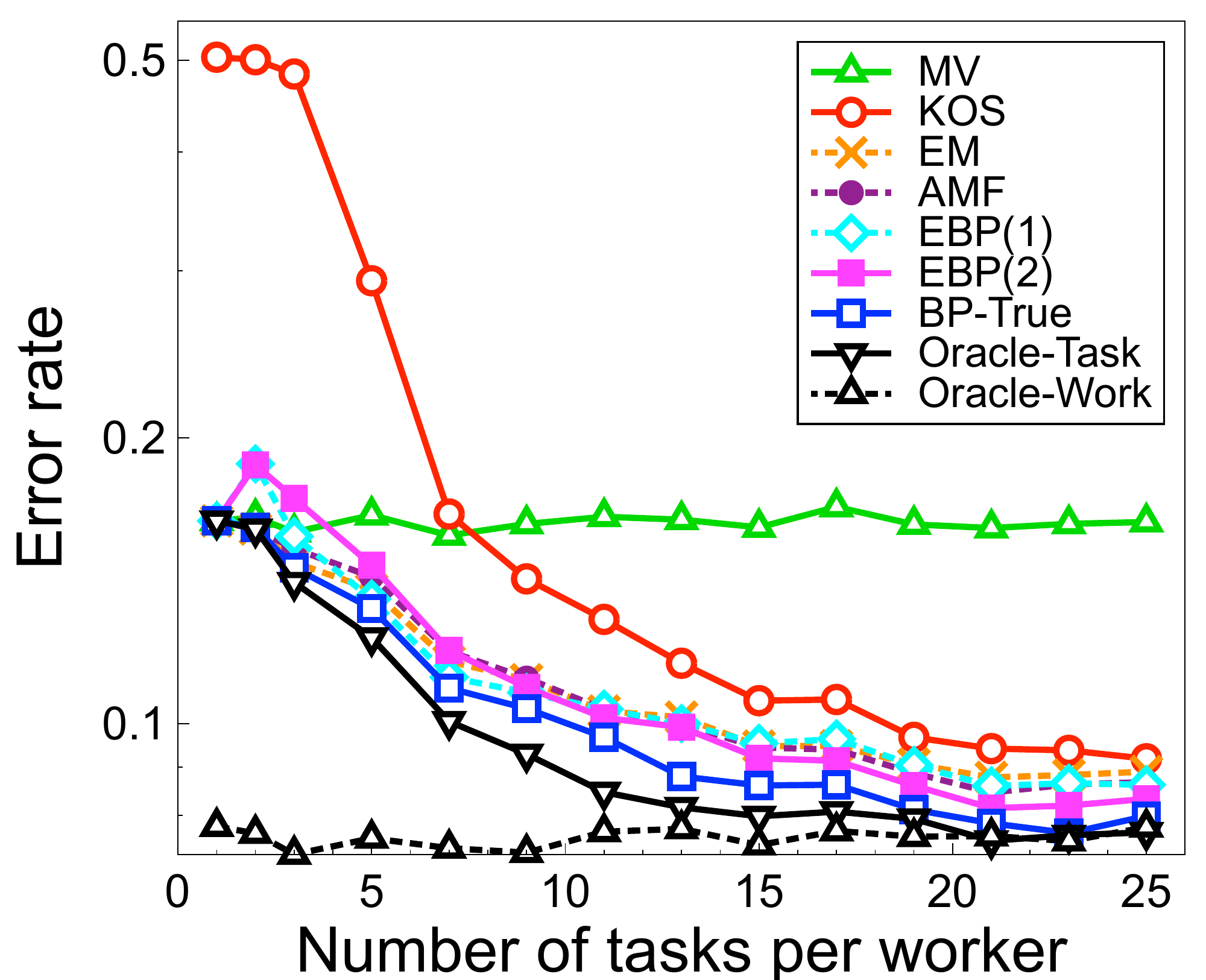}
     } 
   \end{center}
\end{minipage}
\begin{minipage}{.32\textwidth}
     \subfigure[ASH model with $r = 5$\label{fig:ASH_r5}]{\includegraphics[width=1\columnwidth]{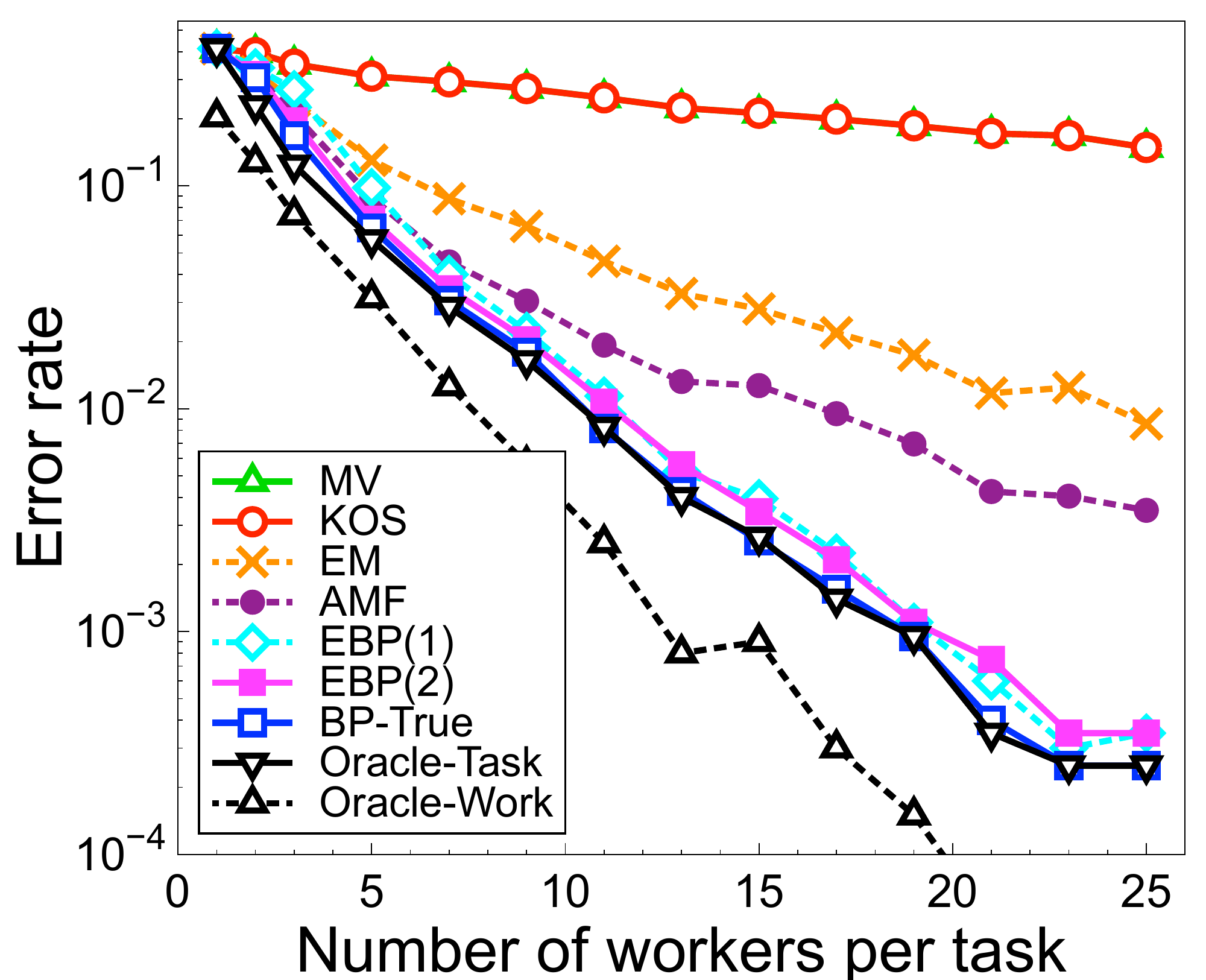}
     }
\subfigure[ASH model with $\ell = 5$\label{fig:ASH_l5}]{\includegraphics[width=1\columnwidth]{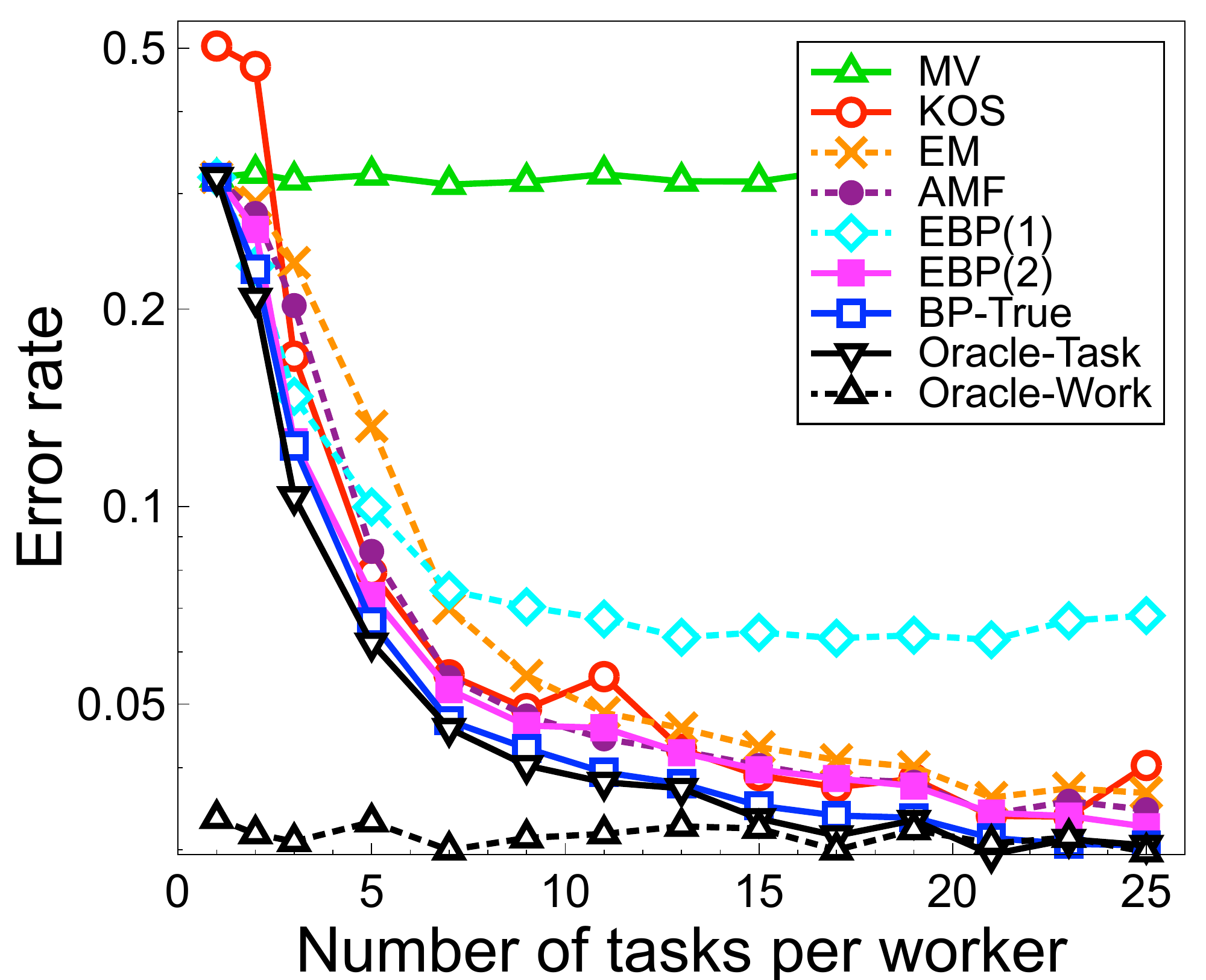}
     } 
\end{minipage}
\begin{minipage}{.32\textwidth}
        \subfigure[{\MF SIM} dataset\label{fig:KOS}]{\includegraphics[width=1\columnwidth]{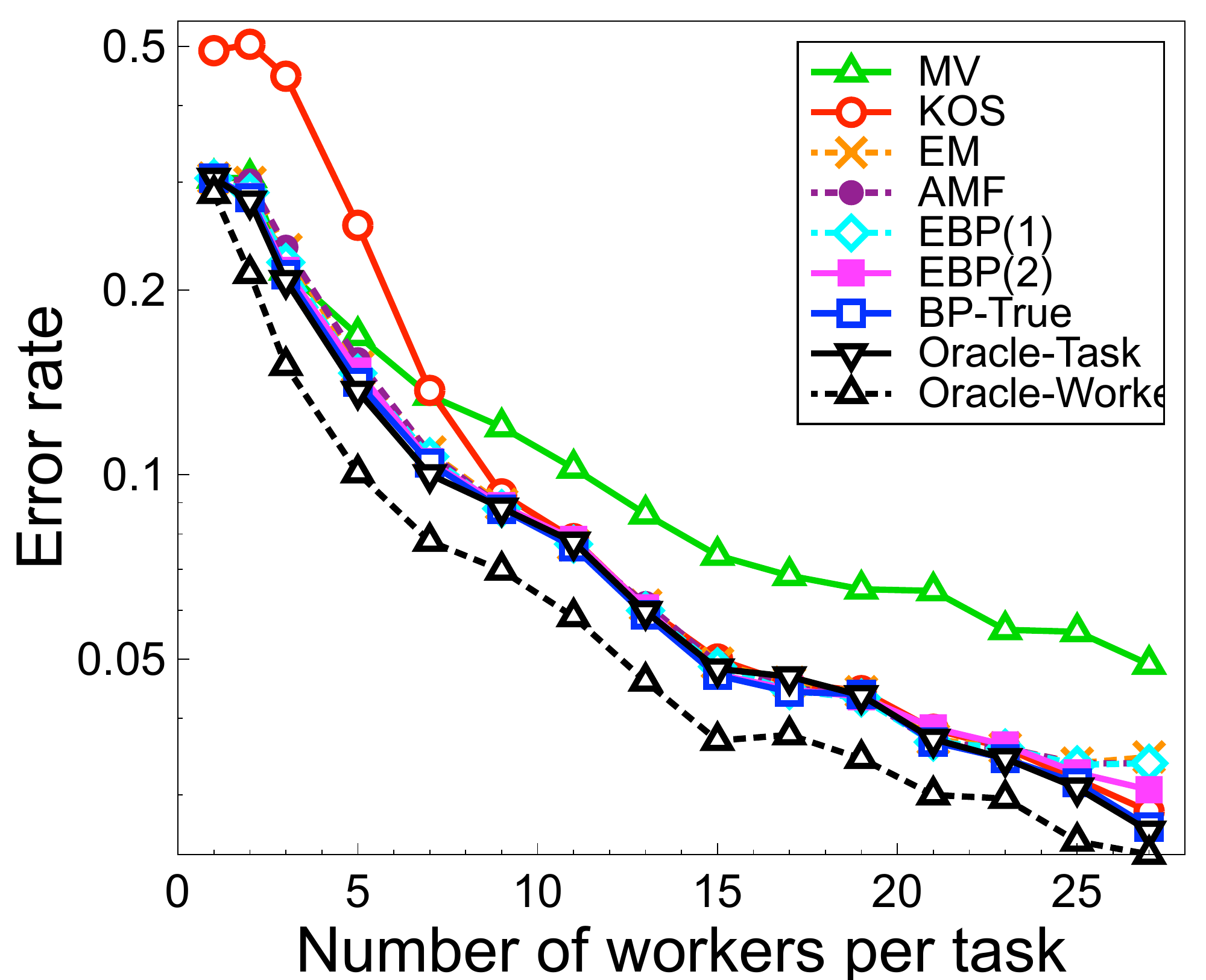}} \\
        \subfigure[{\MF TEMP} dataset\label{fig:TEMP}]{\includegraphics[width=1\columnwidth]{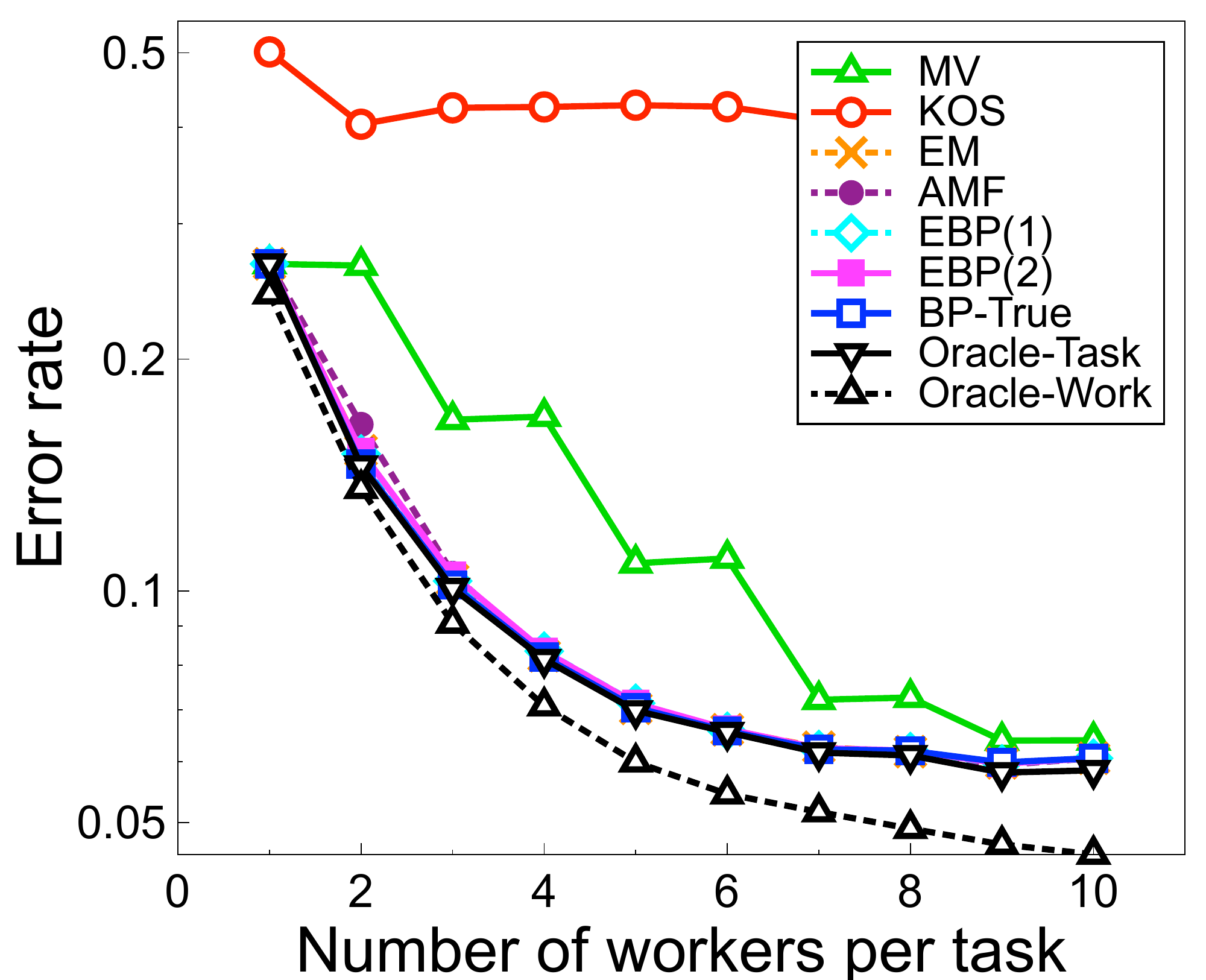}} 
\end{minipage}
   \end{center}
\vspace{-0.1in}
   \caption{The average fraction of incorrectly labeled tasks on the
     synthetic datasets and the real-world Amazon Mechanical Turk
     datasets; (a)-(b) the synthetic datasets consisting of $200$
     tasks with the spammer-hammer (SH) model with $\pi(0.5) =
     \pi(0.9) = 1/2$; (c)-(d) the synthetic datasets consisting of
     $200$ tasks with the adversary-spammer-hammer (ASH) model with
     $\pi(0.1) = \pi(0.5) = 1/4$ and $ \pi(0.9) = 1/2$; (a)
     Color-similarity comparison ({\MF SIM}) dataset with $50$ tasks
     and $28$ workers obtained in \cite{kos2014}; (b) Temporal
     ordering ({\MF TEMP}) dataset with $462$ tasks and $76$ workers
     obtained in \cite{snow2008}.}
\label{fig:all}
\vspace{-0.1in}
\end{figure*}

\subsection{Performance on Synthetic Datasets}
\label{sec:numerical-pf}

We first compare all the algorithms with synthetic datasets generated
by the set of random $(\ell, r)$-regular bipartite graphs having $200$
tasks from the configuration model \cite{bollobas1998}, where we vary
either $\ell$ or $r.$ We randomly choose worker's reliability $p_u$
from the {\it spammer-hammer} model with $\pi(0.5) = \pi(0.9) = 1/2$
and the {\it adversary-spammer-hammer} model with $\pi(0.1) = \pi(0.5)
= 1/2$ and $\pi(0.9) = 1/2$, whose results are plotted in
Figures~\ref{fig:SH_r5}-\ref{fig:SH_l5} and
Figures~\ref{fig:ASH_r5}-\ref{fig:ASH_l5}, respectively.

\vspace{0.05in}
\noindent {\bf Optimality of BP.} We observe that {\MF BP-True} with the knowledge
of the true reliability distribution has the negligible performance gap
from the lower bound of {\MF Oracle-Task}, whereas other algorithms have
the suboptimal performance and their suboptimality gap depends on
$\ell, r$ and the reliability distribution $\pi$ (see
Figures~\ref{fig:ASH_r5}). As discussed in \cite{kos2011}, we observe a
threshold behavior at $(\ell - 1) (r - 1) = 1/q^2$ where for small $\ell$
and $r$ {\MF MV} outperforms {\MF KOS} but for large $\ell$ and $r$ {\MF
  KOS} is better. However, {\MF  BP-true} consistently outperforms all
other algorithms irrespective of the values of  $\ell$ and $r$. 

\vspace{0.05in}
\noindent {\bf Near-optimality of EBP.} Even without knowing the true reliability
distribution, {\MF EBP} with two rounds ({\MF EBP(2)}), achieves almost
the same performance as {\MF BP-True}, as shown in Figure~\ref{fig:ASH_l5}. 
Note that {\MF MV} performs poorly since the number of workers per task is
small and the quality of workers, $\mu = \EXP[2p_u -1]$, is small.
Figure~\ref{fig:ASH_l5} shows that EBP with a single round leads to
moderate performance improvement, but one additional round in {\MF
  EBP(2)} provides us the performance close to optimality. 


\vspace{0.05in}
\noindent {\bf Tighter lower bound.}  We recall that a lower bound in
Lemma~\ref{lem:add} (i.e., {\MF Oracle-Task}) was tight enough to show the exact optimality
of BP, and this tightness is demonstrated in all Figures. 
Note that a different lower
bound is studied by \cite{kos2011} to show just an order-wise
optimality of {\MF KOS}, which is obtained by the Bayesian estimator
with full information on {\it true workers' reliabilities}, marked as
{\MF Oracle-Work} in our plots. 
Both {\MF Oracle-Work} and {\MF Oracle-Task} scale well with respect to
$\ell$ but only {\MF Oracle-Work} does with $r$ as well, thus being a
tighter lower bound (see Figures~\ref{fig:SH_l5}~and~\ref{fig:ASH_l5}).  

\subsection{Performance on Real Datasets}
We use two real-world Amazon Mechanical Turk datasets from
\cite{kos2011} and \cite{snow2008}: {\MF SIM} dataset and {\MF TEMP}
dataset. {\MF SIM} dataset is a set of collected labels where $50$ tasks
on color-similarity comparison are assigned to $28$ users in Amazon
Mechanical Turk. {\MF TEMP} dataset consists of $76$ workers' labels on
$462$ questions about temporal ordering of two events in a collection of
sentences of a natural language. In both datasets, we use the
reliability measured from the dataset as a true workers'
reliability, and we vary $\ell$ by
subsampling the datasets. 
Figures~\ref{fig:KOS} and \ref{fig:TEMP} shows the evaluation results,
where we obtain similar implications to those with the
synthetic datasets, where {\MF EBP(2)} is close to {\MF Oracle-Task} and
outperforms all other the state-of-the-art algorithms. In particular,
{\MF KOS} performs poorly for the {\MF TEMP} dataset, because it is
under the regime for small $\ell$, i.e., before the threshold. 

\section{Conclusion and Discussion}
\label{sec:conclusion}

In this paper, we settle the question of optimality and computational
gap for a canonical scenario for the crowdsourced classification where
the tasks are binary.
Here we discuss some interesting potential extensions of our result.
First the BP optimality can be proved when the task assignment graph is {\it irregular}. 
Our proof of the BP optimality uses the locally tree-like structure in \eqref{eq:tree-probability} 
and the decaying correlation in Lemma~\ref{lem:correlation-decay}. 
These properties hold as long as 
the numbers of workers per task are finite. 
One can potentially generalize 
Theorem \ref{thm:optimality} to irregular bipartite graphs, 
where each task is assigned to sufficiently large but {\it different} number of workers and each worker is assigned to 
large but different number of tasks. 
This extension is important in practical setting where the workers decide how many tasks to work on.

Second it would be interesting to tighten the constants in
the error exponent in \eqref{eq:errorexponent} since the actual
performance of BP is better than predicted by this upper bound.
The
analysis could be significantly tightened, if one can provide tighter
analysis of both the majority voting and the KOS algorithm.
  Next, a
tighter analysis of the oracle error rate is needed. We provide an
oracle estimator that is significantly tighter than the naive oracle
estimators presented in \cite{kos2011}. This strong oracle can be
numerically evaluated, as we do in our experiments.  However, it is not
known how the error achieved by this oracle estimator scales with
problem parameters.  A tight analysis of this lower bound in a form
similar to \eqref{eq:errorexponent} would complete the investigation of
optimality of BP.  Finally, it has been observed in
\cite{kos2011,KOS13SIGMETRICS} that there exists a spectral barrier at
$(\ell-1)(r-1)=1/q^2$, where $q=\EXP[(2p_u-1)^2]$.  Below the spectral
barrier, we observe that the gap between the simple majority voting and
BP becomes narrower as we step away from this threshold.  It is of
interest to identify where MV is optimal, in order to provide guidelines
on how to design crowdsourcing experiments and which algorithms to use.

When we have more than two classes, our algorithm naturally generalizes.
However, the computational complexity increases and the analysis
techniques do not generalize.  We need to investigate other inference
algorithms, perhaps those based on semidefinite programming or
expectation maximization, and provide an analysis that naturally
generalizes to multiple classes. When there are $k$ classes,
characterizing the error rate when $k$ scales as $n^\alpha$ for some
parameter $\alpha$ is of interest. We expect BP to be no longer optimal
for some regimes of $\alpha$.

\updated{
One of the major drawback of the Dawid-Skene model is that it does not
account for tasks that have different difficulty levels.  In real-world
crowdsourcing data, it is common to see some tasks that are more
difficult than the others.  To capture such heterogeneity, several
generalized models have been proposed
\cite{RYZ10,whitehill09,Welinder10,snow2008,sheng2008,ZPBM12,ZLPCS15, KhetanOh2016}.
}
For these general models, the questions of the error rate achieved by
efficient inference algorithms is widely open.  Finally, in real
crowdsourcing systems, adaptive design is common. One can decide to
collect more data on those tasks that are more difficult.  Tighter
analysis of the error rate can provide guidelines on how to design such
adaptive crowdsourcing experiments.
Understanding such adaptive task assignments is an important topic, as they are widely used in practice. 
Under the standard Dawid-Skene model studied in this paper, it is known that there is not much gain in 
using adaptive schemes \cite{kos2014}. 
The main reason is that all tasks are inherently assumed to be equally easy (or difficult) and 
there is not much gain in identifying tasks with less confidence and assigning more workers on those tasks. 
However, recent advances work in  \cite{KhetanOh2016} proves that 
under a more general variation of the Dawid-Skene model, 
it is possible to significantly outperform non-adaptive schemes (such as those studied in this paper), by using adaptive task assignment schemes. 
Understanding the optimality of BP under this more generalized Dawid-Skene model is an interesting open problem. 
It is not even clear how to run BP in this case, as both tasks and workers are parametrized by continuous variables. 

Finally, we note that a preliminary version of this work has been published as \cite{ok2016icml}, 
where the authors showed the BP optimality when $r = 2$. 
In this work, we provide a generalized proof of the BP
optimality with all $r \ge 1$.


 \section*{Acknowledgment}
 
This work is supported by NSF SaTC award CNS-1527754, and NSF CISE award CCF-1553452.

\bibliography{ref}
\bibliographystyle{IEEEtran}

\end{document}